\documentclass[letterpaper]{article} 
\usepackage{aaai24}  
\usepackage{times}  
\usepackage{helvet}  
\usepackage{courier}  
\usepackage[hyphens]{url}  
\usepackage{graphicx} 
\usepackage{natbib}  
\usepackage{caption} 
\usepackage{algorithm}
\usepackage{algorithmic}

\usepackage{newfloat}
\usepackage{listings}
\DeclareCaptionStyle{ruled}{labelfont=normalfont,labelsep=colon,strut=off} 
\floatstyle{ruled}
\newfloat{listing}{tb}{lst}{}
\floatname{listing}{Listing}
\title{Data-Dependent Stability Analysis of Adversarial Training}
\author{%
  Yihan Wang$^{1, 2}$, Shuang Liu$^{1, 2}$, Xiao-Shan Gao$^{1, 2}$
    \thanks{Corresponding author.}}

\affiliations{
\textsuperscript{\rm 1}Academy of Mathematics and Systems Science, Chinese Academy of Sciences, Beijing 100190, China\\
\textsuperscript{\rm 2}University of Chinese Academy of Sciences, Beijing 101408, China\\
\footnotesize{\texttt{yihanwang@amss.ac.cn, 
liushuang2020@amss.ac.cn, 
xgao@mmrc.iss.ac.cn}}
}

\usepackage{bibentry}

\usepackage{subfigure}

\usepackage{amsmath}
\usepackage{amssymb}
\usepackage{mathtools}
\usepackage{amsthm}
\usepackage{color}

\def\B{{\mathcal{B}}}
\def\R{{\mathbb{R}}}
\def\I{{\mathbb{I}}}

\def\N{{\mathbb{N}}}
\def\E{\mathop{\mathbb{E}}}
\def\D{{\mathbb{D}}}
\def\A{{\mathcal{A}}}
\def\Dd{{\mathcal{D}}}

\def\O{{\mathcal{O}}}

\def\P{{\mathcal{P}}}
\def\G{{\mathcal{G}}}
\def\AR{{\mathcal{R}}}

\def\gen{\hbox{\scriptsize\rm{gen}}}
\def\opt{\hbox{\scriptsize\rm{opt}}}

\newtheorem{theorem}{Theorem}

\newtheorem{proposition}[theorem]{proposition}
\newtheorem{lemma}[theorem]{Lemma}
\newtheorem{corollary}[theorem]{corollary}

\newtheorem{remark}[theorem]{Remark}
\newtheorem{definition}[theorem]{Definition}

\newtheorem{assumption}[theorem]{Assumption}

\begin{document}

\maketitle

\begin{abstract}

Stability analysis is an essential aspect of studying the generalization ability of deep learning, as it involves deriving generalization bounds for stochastic gradient descent-based training algorithms. Adversarial training is the most widely used defense against adversarial example attacks. However, previous generalization bounds for adversarial training have not included information regarding the data distribution.
In this paper, we fill this gap by providing generalization bounds for stochastic gradient descent-based adversarial training that incorporate data distribution information. We utilize the concepts of on-average stability and high-order approximate Lipschitz conditions to examine how changes in data distribution and adversarial budget can affect robust generalization gaps. Our derived generalization bounds for both convex and non-convex losses are at least as good as the uniform stability-based counterparts which do not include data distribution information. Furthermore, our findings demonstrate how distribution shifts from data poisoning attacks can impact robust generalization.
\end{abstract}

\section{Introduction}
Deep learning models acquire knowledge from training data and generalize to unseen data. 
Generalization plays a key role in successful machine learning algorithms.
On the other hand, a neural network can be easily fooled by adversarial examples  \cite{szegedy2013intriguing, goodfellow2014explaining}.
Though adversarial training   \cite{madry2017towards} can largely alleviate the adversarial vulnerability of networks,
the corresponding robust generalization is more difficult and 
robust overfitting  \cite{rice2020overfitting} harms the robust performance to a very large degree. 
To understand the generalization ability of adversarial training, an important research direction is to give a theoretical analysis of its generalization bounds.

Algorithmic stability  \cite{bousquet2002stability, shalev2010learnability} can derive the generalization bounds.
In standard training, 
the uniform stability of stochastic gradient descent (SGD) was studied by \citet{hardt2016train},
assuming the loss function is $L$-Lipschitz and $\beta$-gradient Lipschitz.
\citet{bassily2020stability} extended the results to non-smooth convex losses.
\citet{kuzborskij2018data} employed on-average stability and provided data-dependent generalization bounds.
In adversarial training,
the stability of minimax problems is discussed by  \citet{farnia2021train}, 
assuming the inner maximization problem is strongly concave.
\citet{xing2021algorithmic} considered the uniform stability of adversarial training on non-smooth losses.
Under the $\eta$-approximate $\beta$-gradient Lipschitz assumption,
\citet{xiao2022stability} derived generalization bounds for SGD in adversarial training,  which involved the adversarial training budget.

It is generally believed that the difficulty of robust generalization involves three aspects including model capacity, training algorithm, and data distribution.
The capacity of a strictly robust classifier on a well-separated distribution should be exponential in the data dimension \cite{li2022robust}.
The previous generalization bounds based on uniform stability analyses \cite{xing2021algorithmic, xiao2022stability} of adversarial training algorithms did not contain information about data distribution. 
%

In this paper,
we analyze the on-average stability of SGD-based adversarial training and derive data-dependent generalization bounds to illustrate robust generalization, that is, the generalization bounds contain information of data distribution. 
For the convex adversarial losses, 
assuming the losses are Lipschitz and approximately gradient Lipschitz,
we give a generalization bound dependent on the adversarial population risk at the initialization point and the variance of stochastic gradients over the distribution.
Assuming the losses are approximately Hessian Lipschitz in addition,
we provide a generalization bound for the non-convex adversarial losses.
Besides the variance of stochastic gradients over the distribution,
this bound depends on the curvature (the norm of the Hessian matrix) at the initialization point and the population risk at the output parameters.
Our bounds grow with the adversarial training  
budget and cover the standard training setting when the budget becomes zero.
Our bounds for both convex and non-convex losses are no worse than the uniform stability-based counterparts but capture the information about the data distribution and the initialization point.

An additional advantage of our generalization bound over the previous ones is that it describes the effects of distribution shifts caused by data poisoning attacks and hence interprets the shrinkage of generalization gaps in adversarial training under stability attacks since the poisoned distributions can reduce the adversarial population risk over the poisoned data.

The rest of the paper is organized as follows.
In Section \ref{sec:stability-and-generalization},
we revisit the relationship between stability and robust generalization for adversarial training.
In Section \ref{sec:data-dependent-stability-of-at}, 
we provide our main results. 
In Section \ref{sec:experiments}, 
we present experimental results to verify the theoretical results.
All proofs are deferred to Appendix \ref{app:proofs}.

\section{Related Works}\label{sec:related-works}
\textbf{Robust Generalization.}
Machine learning models are highly vulnerable to adversarial examples  \cite{szegedy2013intriguing, biggio2013evasion, nguyen2015deep, moosavi2016deepfool}, where crafted and imperceptible perturbations to input data can easily fool a well-trained classifier.
A widely adopted illustration attributes adversarial examples to the presence of non-robust features \cite{ilyas2019adversarial}. 
Among numerous proposed defenses against adversarial attacks,
adversarial training    \cite{goodfellow2014explaining,shaham2015understanding, madry2017towards} has become a major approach to training a robust deep neural network and can achieve optimal robust accuracy if certain loss functions are used \cite{gao2002StGame}.

The generalization in adversarial training is much more tricky 
than that in standard training and requires more data and larger models \cite{schmidt2018adversarially, gowal2021improving, li2022robust,wang2023better}. 
The robust overfitting  \cite{rice2020overfitting} phenomenon 
harms the robustness in a long training procedure.
In recent years, different methods  \cite{chen2020robust, wu2020adversarial, yu2022understanding, chen2022sparsity} have been proposed to alleviate robust overfitting.

%
%
\textbf{Algorithmic Stability.}
Modern stability analysis goes back to the work of  \citet{bousquet2002stability}.
Notions of stability fall into two categories: data-free and data-dependent ones.
The first category is usually called uniform stability.
Generalization bounds of SGD   have been analyzed using uniform stability 
under Lipschitz and smoothness conditions by  \citet{hardt2016train}.
\citet{bassily2020stability} extended the results to the non-smooth convex case. 
\citet{farnia2021train} studied the role of optimization algorithm in the generalization performance of the minimax model.
The uniform stability of adversarial training  
has been reported by  \citet{ xing2021algorithmic, xiao2022stability}.
The data-dependent stability  \cite{kuzborskij2018data} employing the notion of  on-average stability  \cite{shalev2010learnability} focused on the stability of the SGD-based standard training
under the data distribution given an initialization point.
%

\textbf{Data Poisoning.}
As defensive strategies against unauthorized exploitation of personal data,
availability attacks \cite{huang2021unlearnable, fowl2021adversarial, feng2019learning,  liu2021going, ren2022transferable} perturb the training data imperceptibly 
such that the trained models learn nothing useful and become futile.
Adversarial training can mitigate such availability attacks  \cite{tao2021better}.
Stability attacks
\cite{tao2022can, fu2022robust, wang2021fooling}
have been proposed to come through adversarial training and result in a large degradation in the robust test performance.
%
%
The shortcut interpretation \cite{yu2022availability} suggests that stability poisoning attacks root "easy-to-learn" features in the poisoned training data. However, these features do not appear in clean data.
Our generalization bound can be used to interpret the shrinkage of generalization gaps in adversarial training under stability attacks,
since it contains information of data distribution.

\section{Preliminaries}
\label{sec:stability-and-generalization}
In this section, we revisit the robust generalization gap and the on-average stability analysis. 

\subsection{Robust Generalization Gap}\label{sec:preliminaries}
Let  $\Dd$ be a data distribution over an image classification data space $\D=[0,1]^d\times[m]$, where $[0,1]^d$ contains the image space and $[m]=\{1,\ldots,m\}$ is the label set.
A data set $S$ of $n$ samples is  drawn i.i.d. from $\Dd$ and is denoted by $S\sim \Dd^n$. 
Given a network with parameter $\theta$ and a non-negative loss function $l(\theta, z):\R^k\times \D\to \R_{\ge0}$,  the standard training minimizes the {\em empirical risk}
$\E_{z\in S}\,l(\theta, z)$ with SGD.
 
\textbf{Adversarial Training.}
As a major defense approach,
adversarial training  \cite{madry2017towards} 
refers to a bi-level optimization, 
of which the inner maximization iteratively searches for 
the strongest perturbation inside a $L_p$-norm ball and the outer minimization optimizes the model via the loss on the perturbed data.
Formally, given an adversarial budget $\epsilon$,
the adversarial training uses the {\em adversarial  loss}:
\begin{align*}
    h(\theta, z) = \max_{z' \in \B_\epsilon(z)} l(z',\theta),
\end{align*}
where $\B_\epsilon(z)=\{z'\in \D:||z'-z||_p\leq \epsilon \}$
and $p\in\N\cup\{\infty\}$.
Here the $p$-norm is for the image part of $z$.
When $\epsilon=0$, we have $h=l$ and adversarial training reduces to standard training.
The {\em adversarial population risk} and {\em adversarial empirical risk}
are respectively defined as 
\begin{align*}
    \AR_\Dd (\theta)=\E_{z\sim \Dd}[ h(\theta, z)] 
    \text{\ \ \  and\ \ \  }
    \AR_S(\theta)=\E_{z\in S}[ h(\theta, z)].
\end{align*}
We denote the SGD algorithm of adversarial training  by $\A$,
which inputs a training set $S$ 
and outputs the parameter $\A(S)$ of a network through
minimizing the adversarial empirical risk $\AR_S$.
%
 
\textbf{Robust Generalization Gap.}
Let $\theta^*,\Bar{\theta}$ be the optimal solutions of learning over $\Dd$ and $S$, namely minimizing $\AR_\Dd (\theta)$ and $\AR_S(\theta)$, respectively. 
Then, for the output
$\hat{\theta}=\A(S)$ of algorithm $\A$, the excess risk can be decomposed as 
\begin{align*}
    \AR_\Dd(\hat{\theta}) - \AR_\Dd(\theta^*)
    &= \underbrace{\AR_\Dd(\hat{\theta})-\AR_S(\hat{\theta})}_{\varepsilon_{\gen}}
    + \underbrace{\AR_S(\hat{\theta}) - \AR_S(\Bar{\theta})}_{\varepsilon_{\opt}}\\
    +& \underbrace{\AR_S(\Bar{\theta}) - \AR_S(\theta^*)}_{\leq 0}
    + \underbrace{\AR_S(\theta^*) - \AR_\Dd(\theta^*)}_{\E = 0}.
\end{align*}
To control the excess risk, 
we need to control the {\em robust generalization gap} $\varepsilon_{\gen}$ and the {\em robust optimization gap} $\varepsilon_{\opt}$.
The robust optimization gap in adversarial training   
has been studied a lot theoretically  \cite{nemirovski2009robust,xiao2022stability}. 
Also, empirical results  \cite{madry2017towards, zhang2019theoretically, wang2019improving, wu2020adversarial} present narrow robust optimization gaps.
 
On the other hand, robust overfitting  \cite{rice2020overfitting}
is a dominant phenomenon in adversarial training   
that hinders deep neural networks from attaining high robust performance. 
Hence, we focus on the robust generalization gap $\varepsilon_{\gen}$ in this paper.






\subsection{On-Average Stability}
In order to analyze the data-dependent stability,
we employ the notion of on-average stability.
Given a data set $S=\{z_1,\cdots,z_n\}\sim \Dd^n$ 
and    $z\sim \Dd$, 
replacing $z_i$ in $S$ with $z$,
we denote $S^{i,z}=\{z_1,\cdots, z_{i-1},z,z_{i+1},\cdots, z_n \}$ with $i\in [n]$.
\begin{definition}[On-Average Stability]\label{def:on-average-stability}
    A randomized algorithm $\A$ is {\em $\varepsilon$-on-average stable} if 
    \begin{align}
        \sup_{i\in [n]} \E_{S,z,\A}[h(\A(S),z) - h(\A(S^{i,z}),z)]\leq \varepsilon,\label{equ:on-average-stability}
    \end{align}
    where $S\sim \Dd^n$, $z\sim \Dd$, and $\varepsilon$ can depend on the data distribution $\Dd$ and the initialization point of $\A$.
\end{definition}
The on-average stability considers the expected difference between the losses of algorithm outputs on $S$ and its replace-one-example version for all replacement index $i$. 
The on-average stability derives the upper bound of the generalization gap as follows.

\begin{theorem}[\citet{kuzborskij2018data}]
\label{thm-KL}
{If $\A$ is $\varepsilon$-on-average  stable, then the robust generalization gap of $\A$ is bounded by $\varepsilon$:}
    \begin{align*}
        \E_{S,\A}[\AR_\Dd(\A(S)) - \AR_S(\A(S))]\leq \varepsilon.
    \end{align*}
\end{theorem}



\section{Theoretical Results}\label{sec:data-dependent-stability-of-at}
In this section, we give the data-dependent stability analysis of adversarial training   
for both convex and non-convex adversarial losses.
We provide proof sketches of our results and the full proofs are placed in Appendix \ref{app:proofs}.
\subsection{Lipschitz Conditions}\label{subsec:lip-conditions}
Stability analysis always relies on some Lipschitz conditions.
The loss function is assumed to be $L$-Lipschitz and $\beta$-gradient Lipschitz, i.e. $\beta$-smooth in the work  \cite{hardt2016train}.
For  adversarial training, 
we need the adversarial loss $h(\theta,z)$ to satisfy some Lipschitz conditions.
It is not reasonable to directly endow $h$ with Lipschitz conditions,
since $h(\theta, z)$ takes the maximum of $l(\theta, z')$ with $z'\in \B_\epsilon(z)$. 
Instead, we assume that the original loss function $l(\theta, z)$ satisfies the following Lipschitz conditions.
%
%
 %
Let $||\cdot||_p$ be the $p$-norm of vectors or matrices and we write $||\cdot||$ instead of $||\cdot||_2$ for brevity.
In this paper, $\nabla$ is the abbreviation for $\nabla_\theta$.
 \begin{assumption}\label{ass:l-lip} The loss $l$ is $L$-Lipschitz in $\theta$:
            \begin{align*}
            ||l(\theta_1, z)-l(\theta_2, z)|| &\leq L ||\theta_1 - \theta_2||.
            \end{align*}
       \end{assumption} 
\begin{assumption}\label{ass:l-gradient-lip} The loss $l$ is $L_\theta$-gradient Lipschitz in $\theta$ and $L_z$-gradient Lipschitz in $z$:
            \begin{align*}
            ||\nabla l(\theta_1, z) - \nabla l(\theta_2, z)|| & \leq L_\theta ||\theta_1 - \theta_2||,\\
            ||\nabla l(\theta, z_1) - \nabla l(\theta, z_2)|| & \leq L_z ||z_1 - z_2||_p.
            \end{align*}
\end{assumption}        
\begin{assumption}\label{ass:l-lip-hessian} The loss $l$ is $H_\theta$-Hessian Lipschitz in $\theta$ and $H_z$-Hessian Lipschitz in $z$:
            \begin{align*}
             ||\nabla^2 l(\theta_1,z) - \nabla^2 l(\theta_2, z)|| &\leq H_\theta ||\theta_1-\theta_2||,\\
             ||\nabla^2 l(\theta,z_1) - \nabla^2 l(\theta, z_2)|| &\leq H_z ||z_1-z_2||_p.\\
            \end{align*}
\end{assumption}
\begin{remark}
For commonly used losses and ReLU-based networks,  Assumption \ref{ass:l-lip} is valid \cite{gao2002StGame}.
The gradient Lipschitz conditions (Lipschitz smoothness) are often used in robustness analysis \cite{sinha2017certifying,liu2020loss, xiao2022stability}. 
Lipschitz Hessians are used in the analysis of SGD \cite{ge2015escaping, kuzborskij2018data}.
Assumptions \ref{ass:l-gradient-lip} and \ref{ass:l-lip-hessian} are valid for networks based on smooth activation functions such as Sigmoid and smooth loss functions such as cross-entropy (CE) and mean squared error (MSE);
related works on ReLU-based networks were given in \cite{allen2019convergence, du2019gradient}.

\end{remark}

Note that the adversarial vulnerability of deep networks is rooted in the explosion of the Lipschitz constant of $l(\theta, z)$ in $z$.
However, the zero-order Lipschitz constant in $\theta$ can be directly inherited by $h(\theta,z)$ \cite{liu2020loss}.
Additional Lipschitz conditions in $z$ imply approximate gradient and Hessian Lipschitz conditions in $\theta$ which are needed for stability analysis.

{
\begin{definition}
    Let $\eta, \beta, \nu, \rho > 0$ and $h(\theta)$ be a second-order differentiable function.
    \begin{enumerate}
        \item $h$ is $\eta$-approximately $\beta$-gradient Lipschitz, if
        \begin{align*}
            ||\nabla h(\theta_1)-\nabla h(\theta_2)|| \leq \beta ||\theta_1 - \theta_2|| + \eta.
        \end{align*}
        \item $h$ is $\nu$-approximately $\rho$-Hessian Lipschitz, if 
        \begin{align*}
            ||\nabla^2 h(\theta_1)-\nabla^2 h(\theta_2)|| \leq \nu ||\theta_1 - \theta_2|| + \rho.
        \end{align*}
    \end{enumerate}
\end{definition}
}

\begin{lemma} \label{lem:h-lip-conditions}
    The adversarial loss $h(\theta, z)$ inherits (approximate)
    Lipschitz   properties from the original loss $l(\theta,z)$. 
    \begin{enumerate}
    
        \item  Under Assumption \ref{ass:l-lip}, $h$ is $L$-Lipschitz with respect to $\theta$:
            \begin{align*}
                ||h(\theta_1, z)-h(\theta_2, z)|| &\leq L ||\theta_1 - \theta_2||.
            \end{align*}
            
        \item  Under  Assumption \ref{ass:l-gradient-lip}, $h$ is $2\epsilon L_z$-{\em approximately $L_\theta$-gradient Lipschitz} with respect to $\theta$:
            \begin{align*}
                ||\nabla h(\theta_1, z)- \nabla h(\theta_2, z)||
                &\leq L_\theta||\theta_1-\theta_2|| + 2\epsilon L_z.
            \end{align*}
        \item Under  Assumption \ref{ass:l-lip-hessian},  $h$ is $2\epsilon H_z$-{\rm approximately $H_\theta$-Hessian Lipschitz} with respect to $\theta$:
            \begin{align*}
                ||\nabla^2 h(\theta_1,z) - \nabla^2 h(\theta_2, z)|| 
                &\leq H_\theta 
                ||\theta_1-\theta_2|| + 2\epsilon H_z.
            \end{align*}
    \end{enumerate}
\end{lemma}

\subsection{Preliminaries for Analysis} 
\label{subsec:analysis}
We consider the SGD without replacement,
that is, given a training set $S\sim \Dd^n$,
algorithm $\A$ chooses a random permutation $\pi$ over $[n]=\{1,\cdots,n\}$
and cycles through $S$ in the order determined by the permutation.
If not mentioned otherwise, our analyses focus on the on-average stability of adversarial training in a single pass.

Suppose the update of $\A$ starts from an initialization point $\theta_1$ and for $t\in [n]$,
\begin{align*}
    \theta_{t+1} = \G_\A(\theta_t, z_{\pi(t)}, \alpha_t),
\end{align*} 
where the permutation $\pi$ depends on $\A$
and $\alpha_t$ is the $t$-th step size.
We update $T$ steps in a single pass for $T\in [n]$ and analyze the on-average stability of the algorithm output $\A(S) = \theta_{T+1}$.
We assume the variance of stochastic gradients in $\A$ obey 
\begin{equation}
\label{eq-var}
    \E_{S}[||\nabla h(\theta_t, z_{\pi(t)}) - \nabla \AR_\Dd(\theta_t)||^2] \leq \sigma^2
\end{equation}
for all $t\in[T]$.
The variance $\sigma$ describes the distance between the stochastic gradient and the optimal gradient.
{Indeed, $\sigma$ will change if the distribution $\Dd$ changes.}

\subsection{Convex Adversarial Losses}
For convex adversarial losses, 
our analysis requires the approximate gradient Lipschitz assumption.
\begin{theorem} \label{thm:convex-main}
    Assume the adversarial loss $h(\theta, z)$ is 
    non-negative, convex in $\theta$, $L$-Lipschitz and $\eta$-approximately $\beta$-gradient Lipschitz with respect to $\theta$. 
    Let the step sizes $\alpha_t\leq 1/\beta$.
    Then algorithm $\A$ is $\varepsilon(\Dd, \theta_1)$-on-average stable with
    \begin{align}
        &\varepsilon(\Dd, \theta_1)
        =(\frac{2\sigma L}{n}+L\eta)\sum_{t=1}^{T}\alpha_t 
        +  \frac{4L}{n}\sqrt{\sum_{t=1}^{T}\alpha_t}\nonumber\\ 
        &\cdot \sqrt{\AR_\Dd(\theta_1)-\AR_\Dd(\theta^*) + \frac{\beta\sigma^2}{2} \sum_{t=1}^{T}\alpha_t^2 + \eta L \sum_{t=1}^{T}\alpha_t},\label{equ:main-thm-convex}
    \end{align}
    where $\theta_1$ is the initialization point.
\end{theorem}
\textbf{Proof sketch.}
Given a data set $S=\{z_1,\cdots,z_n\}\sim \Dd^n$, 
an example $z\sim \Dd$, and an index $i\in [n]$, 
we denote $S^{i,z} = \{z'_1,\cdots,z'_n\}$ with $z'_j=z_j$ for $j\neq i$ and $z'_i=z$.
Let $\theta_t$, $\theta_t'$ be the $t$-th outputs of $\A(S)$ and $\A(S^{i,z})$, respectively.
Denote the distance of two trajectories at step $t$ by 
$\delta_t(S, z, i, \A) = ||\theta_t-\theta_t'||$.
As both two updates start from $\theta_1$, we have $\delta_1(S,z, i,\A)=0$.
Denote $\Delta_t(S,z,i)=\E_\A[\delta_t(S,z,i, \A)|\delta_{t_0}(S,z,i,\A)=0]$.
Lemma \ref{lem:kuz-average-stable-delta}  (Lemma 5 in \cite{kuzborskij2018data}) tells that
\begin{align*}
    \E_{S, z,\A}[h(\theta_t, z)-h(\theta_t', z)]
    \leq L \E_{S,z}[\Delta_t(S,z,i)].
\end{align*}
According to whether the algorithm meets the different sample with index $i$ at step $t$, we derive the following recursion formula involving the adversarial budget $\eta=2\epsilon L_z$,
\begin{align*}
    \Delta_{t+1}(S,z,i) 
    &\leq \Delta_t(S,z,i) + (1-\frac{1}{n})\alpha_t \eta \\
    &+ \frac{\alpha_t}{n}\E_\A[||\nabla h(\theta_t,z_{\pi(t)})||+||\nabla h(\theta_t',z'_{\pi(t)})||].
\end{align*}
By repeatedly applying Jensen's inequality, both expectations 
$\E_{S}[\sum_{t=1}^{T}\alpha_t||\nabla h(\theta_t,z_{\pi(t)})||]$ and  
$\E_{z,S}[\sum_{t=1}^{T}\alpha_t||\nabla h(\theta_t',z'_{\pi(t)})||]$ have the same upper bound (Lemma \ref{lem:sum-grad-h-inequality}) 
\begin{align*}
    &\sum_{t=1}^{T}\sigma\alpha_t +  2\sqrt{\sum_{t=1}^{T}\alpha_t} \cdot \sqrt{r + \frac{\beta}{2} \sum_{t=1}^{T} \sigma^2\alpha_t^2 + \eta L \sum_{t=1}^{T}\alpha_t},
\end{align*}
where $r = \AR_\Dd(\theta_1)-\AR_\Dd(\theta^*)$.
Then, we can recursively bound $\E_{S, z,\A}[h(\theta_t, z)-h(\theta_t', z)]$ and 
prove the theorem.

\begin{remark}
By Theorem \ref{thm-KL},  Theorem \ref{thm:convex-main} gives an upper bound  $\varepsilon(\Dd, \theta_1)$ for the robust generalization gap of algorithm $\A$, 
similarly for Theorems \ref{thm:non-convex-main}
and \ref{thm:main-non-convex-multiple-pass}.
\end{remark}
When the step size is $\alpha_t=\O(\frac{1}{\sqrt{t}})\leq \frac{1}{\beta}$ and the adversarial budget $\epsilon=0$,
this bound reduces to the result in  \cite{kuzborskij2018data}.
Now we fix step sizes to be constant and bound the adversarial loss gap between the initialization point and the optima via the Lipschitz condition.
\begin{corollary} \label{cor:convex-main-constant}
    Let the step size $\alpha_t $ be a constant $\alpha \leq 1/\beta$ and $r = \AR_\Dd(\theta_1)-\AR_\Dd(\theta^*)$.
    Then   algorithm $\A$ is $\varepsilon(\Dd, \theta_1)$-on-average stable with
\begin{align}
        &\varepsilon(\Dd, \theta_1)\nonumber \\
        =&\eta\alpha  L T + \frac{2\alpha L T}{n}(\sigma+\sqrt{2}\sigma + 2\sqrt{\eta  L}) + \frac{4L\sqrt{\alpha rT}}{n}. \label{equ:main-result-convex}
\end{align}
\end{corollary}



\textbf{Comparison.}
We compare our result with existing results for adversarial training with convex adversarial losses in a single pass.
For clarity, we take constant step size $\alpha$ and use the $\O$ notation.
\begin{itemize}
    \item Result of  \citet{xing2021algorithmic}:
    \begin{align}
        \O(\alpha  L^2 \sqrt{T}+ \frac{\alpha L^2 T }{n}). \label{equ:result-xing-convex}
    \end{align}
    
    \item Result of  \citet{xiao2022stability}:
    \begin{align}
        \O(\eta  \alpha L  T  + \frac{\alpha L^2 T }{n}). \label{equ:result-xiao-convex}
    \end{align}
    
    \item Our result:
    \begin{align}
        \O(\eta\alpha T L+ 
        \frac{\alpha \sigma L T  + \alpha  \sqrt{\eta}L^{1.5} T  + L \sqrt{\alpha r T  }}{n} ) \label{equ:result-ours-convex}
    \end{align}
\end{itemize}
The smoothness of $h$ is not required for the result of  \cite{xing2021algorithmic}.
The bound~\eqref{equ:result-xing-convex} remains unchanged under changes in the adversarial training budget $\epsilon$.
Thus this result does not capture the empirical observations that the robust overfitting phenomenon deteriorates as $\epsilon$ grows.
Approximate smoothness of $h$ is required for the result of  \cite{xiao2022stability}.
The bound~\eqref{equ:result-xiao-convex} takes into account $\epsilon$, i.e. $\eta = 2\epsilon L_z$ by the second statement in Lemma \ref{lem:h-lip-conditions}.
However, this bound stays unchanged whenever the distribution shifts or the initialization point changes. 
Detailed discussion is shown in Appendix \ref{app:discussion-of-uniform-stability-based-results}.

Our bound grows with the adversarial training budget $\epsilon$ as well.
In general case, \eqref{equ:result-xiao-convex} and \eqref{equ:result-ours-convex} are both $\O(T)$.
When $\eta=0$, our bound reduces to $\O(\frac{\alpha \sigma L T + L\sqrt{\alpha r T}}{n})$ which is the case for standard training.
In the case that $\eta=0$ and $\sigma$ is negligible, our bound is dominated by the term $\frac{L \sqrt{\alpha r T} }{n}$ and becomes tighter than $\O(T)$ in Equ.~\eqref{equ:result-xing-convex} and ~\eqref{equ:result-xiao-convex}. 
Since  $r$ relies on $\theta_1$ and $\Dd$,
our result implies that a properly selected initialization point matters for robust generalization, and a potential distribution shift caused by some poisoning attack may affect the robust generalization.

\subsection{Non-Convex Adversarial Losses}
For non-convex adversarial losses,
our analysis requires both approximate gradient Lipschitz and approximate Hessian Lipschitz assumptions. 
\begin{theorem}\label{thm:non-convex-main}
    Suppose the adversarial loss $h(\theta,z)$ is
    non-negative,
    $L$-Lipschitz, $\eta$-approximately $\beta$-gradient Lipschitz and 
    $\nu$-approximately $\rho$-Hessian Lipschitz 
    with respect to $\theta$.
    Let the step sizes $\alpha_t=\frac{c}{t}$ with $c\leq \min \{\frac{1}{\beta},\frac{1}{4\beta \ln T}, \frac{1}{8(\beta\ln T)^2}\}$.
    Then $\A$ is $\varepsilon(\Dd, \theta_1)$-on-average stable with 
    \begin{align}
        &\varepsilon(\Dd, \theta_1)\nonumber\\
        =& \frac{1+\frac{1}{c\gamma}}{n}(2c L^2+ n c \eta L)^{\frac{1}{1+c\gamma}}
        \cdot(\E_{S,\A}[\AR_\Dd(\A(S))] T)^{\frac{c\gamma}{1+c\gamma}}, \label{equ:main-thm-non-convex}
    \end{align}
    where
    \begin{align*}
        \gamma
        &=\min\{\beta, \Tilde{\O}(\E_{z}[||\nabla^2 h(\theta_1,z)|| ] + \nu + \Delta^*) \},\\
        \Delta^*
        &=\rho (\sqrt{(\AR_\Dd(\theta_1)-\AR_\Dd(\theta^*))c} + c\sigma + c\sqrt{\eta L}).
    \end{align*}
\end{theorem}

\textbf{Proof sketch.}
By Lemma \ref{lem:kuz-average-stable-delta}  (Lemma 5 in \cite{kuzborskij2018data}), $\forall t_0\in[n+1]$),
\begin{align*}
&\E_{S, z,\A}[h(\theta_{T+1}, z)-h(\theta_{T+1}', z)]\\
&\leq L \E_{S,z}[\Delta_{T+1}(S,z,i)]+\frac{t_0-1}{n}\E_{S,\A}[\AR_\Dd(\theta_{T+1})].
\end{align*}
The key is to recursively bound $\Delta_{T+1}(S,z,i)$.
When the algorithm meets the different sample with index $i$ at step $t$ with probability $\frac{1}{n}$, we have
\begin{align*}
    ||\G_\A(\theta_t) - \G_\A(\theta'_t)||\leq \delta_t(S,z,i,\A) + 2\alpha_t L.
\end{align*}
Otherwise, the second statement in Lemma \ref{lem:ref-xiao-approx-grad-lip} (from \cite{xiao2022stability}) implies \begin{align*}
    ||\G_\A(\theta_t) - \G_\A(\theta'_t)||\leq (1+\alpha_t \beta)\delta_t(S,z,i,\A)+\alpha_t \eta.
\end{align*}
Additionally, in this case, Lemma \ref{lem:non-convex-xi} starts from Taylor expansion with integral remainder and exploits the approximate Hessian Lipschitz condition, deriving another bound as 
\begin{align*}
||\G_\A(\theta_t) - \G_\A(\theta'_t)|| \leq (1+\alpha_t \xi_t(S,z,i,\A))\delta_t(S,z,i,\A),
\end{align*}
where 
\begin{align*}
    \E_{S,z}[\xi_t(S,z,i,\A)]=\Tilde{\O}(\E_{z}[||\nabla^2 h(\theta_1,z)|| ] + \nu + \Delta^*).
\end{align*}
Let $\psi_t(S,z,i) =\E_{\A}[ \min\{\xi_t(S,z,i,\A),\beta \}]$
and we have
\begin{align*}
    &\Delta_{t+1}(S,z,i)\leq \frac{1}{n} (\Delta_t(S,z,i)+2\alpha_t L) \\
    &+ (1-\frac{1}{n}) ((1+\alpha_t \psi_t(S,z,i))\Delta_t(S,z,i) + \alpha_t \eta).
\end{align*}
Assigning proper step sizes $\alpha_t$ and leveraging Lemma \ref{lem:kuz-bernstein}, the on-average stability is given as
\begin{align*}
    &\E_{S, z,\A}[h(\theta_{T+1}, z)-h(\theta_{T+1}', z)]\\
    \leq &(\frac{2L^2+\eta n L}{2n\gamma})(\frac{T}{t_0-1})^{2c\gamma}
    +\frac{t_0-1}{n}\E_{S,\A}[\AR_\Dd(\theta_T)].
\end{align*}
Then we take the optimal $t_0$ and obtain the theorem.

From Equ. \eqref{equ:main-thm-non-convex}, 
we see that smaller $\gamma$ yields higher stability.
Note that $\gamma$ is controlled by $\eta$ and $\nu$, 
the adversarial population risk at the initialization point,
and the average Hessian norm of adversarial loss at the initialization point over the distribution. 

Since  SGD in a single pass is considered, 
we take $T \approx n$ and obtain that 
$\varepsilon(\Dd, \theta_1) = \O(n^{-\frac{1}{1+c\gamma}})$
which can be improved to a more optimistic result $\O(n^{-1})$ 
when the adversarial empirical risk $\AR_S(\A(S))$ becomes negligible according to  \citet{kuzborskij2018data}.
Due to $\eta=2 \epsilon L_z$ and $\nu=2\epsilon H_z$, 
a large adversarial budget $\epsilon$ makes the algorithm unstable and setting $\epsilon=0$ derives the result for standard training.
The gradient and Hessian Lipschitz constants $L_z$ and $H_z$ amplify the effect of $\epsilon$ and this explains why adversarial training appears to be more tricky than standard training and requires more training data  \cite{schmidt2018adversarially, gowal2021improving, wang2023better}.

The initialization point is another factor that affects robust generalization. 
Intuitively, adversarial training prefers an initialization point naturally with low adversarial population risk which is close to the global optima. 
Furthermore, our result suggests that a proper selection of the initialization point should better have low curvature over the distribution.

\textbf{Comparison.}
Assume that the adversarial loss  $h(\theta, z)$ is bounded in $[0, B]$ and $\alpha_t=\frac{c}{t}$ with $c\leq \frac{1}{\beta}$. 
The result\footnote{They reported a conservative result in the paper.
Here we place their optimal result for comparison.} 
of  \citet{xiao2022stability} for the non-convex case is 
\begin{align}
\frac{1+\frac{1}{c\beta}}{n}(2cL^2+nc\eta L)^{\frac{1}{1+c\beta}}(B T)^{\frac{c\beta}{1+c\beta}}.\label{equ:result-xiao-nonconvex}
\end{align}
Observe that \eqref{equ:main-thm-non-convex} and \eqref{equ:result-xiao-nonconvex} have similar forms.
Nevertheless, 
\eqref{equ:result-xiao-nonconvex} remains unchanged under data poisoning attacks.
Our result replaces $\beta$ with $\gamma$ which captures much more information dependent on the initialization point, the loss function, and data distribution.
Besides $\eta$, the approximation $\nu$ emphasizes the effect of $\epsilon$ again in our bound.
Moreover, $\gamma$ and $c$ are bounded by $\beta$ and $\frac{1}{\beta}$ respectively in \eqref{equ:main-thm-non-convex}.
{During the training, the dataset size $n$ is fixed and the term involving the training step $T$ dominates the bound in Equation (8), namely smaller $\gamma$ means smaller $(\E_{S,\A}[\AR_\Dd(\A(S))] T)^{\frac{c\gamma}{1+c\gamma}}$, and then a tighter bound.}
Thus, our result is no worse than \eqref{equ:result-xiao-nonconvex}.

\textbf{Multiple-pass Case.}
Note that Equ. \eqref{equ:main-thm-non-convex} holds within one pass through the training set.
If we loosen some data-dependency requirements, say $\gamma$, 
the on-average stability analysis provides a result for the multiple-pass case.
\begin{theorem}[Multiple-pass Case] 
\label{thm:main-non-convex-multiple-pass}
    Assume the adversarial loss $h(\theta, z)$ is     non-negative, convex in $\theta$, $L$-Lipschitz and $\eta$-approximately $\beta$-gradient Lipschitz with respect to $\theta$. 
    Let the step sizes $\alpha_t\leq \frac{c}{t}$ with $c\leq \frac{1}{\beta}$.
    Then algorithm $\A$ is $\varepsilon(\Dd, \theta_1)$-on-average stable with
    \begin{align}
    &\varepsilon(\Dd, \theta_1)\nonumber\\
    =&\frac{1+\frac{1}{c\beta}}{n}(2cL^2+nc\eta L)^{\frac{1}{1+c\beta}} ( \E_{S,\A}[\AR_\Dd(\A(S))] T)^{\frac{c\beta}{1+c\beta}}.
    \label{equ:multipass-main-result-non-convex}
\end{align}
\end{theorem}
Both 
\eqref{equ:main-thm-non-convex} and \eqref{equ:multipass-main-result-non-convex} 
contain the data-dependent factor $\E_{S,\A}$ $[\AR_\Dd(\A(S))]$ which can be much smaller than $B$ in \eqref{equ:result-xiao-nonconvex}.

\subsection{Poisoned Generalization Gap} \label{subsec:poison-generalization-gap}
To have a closer look at how changes in data distribution can affect robust generalization, we consider the distribution shift caused by a poisoning attack.
A data poisoning attack $\P$ maps a distribution $\Dd$ 
to the poisoned distribution $\P_\#\Dd$.
Poisoning is usually constrained by a given {\em poisoning budget}
$\epsilon'$ such that
$\sup_{z\in \D } ||\P(z)-z||_p \leq \epsilon'$.
The poisoned version of an algorithm $\A$ 
is denoted by $\A_{\P}$ which inputs $S\sim \Dd^n$ and
outputs $\A_{\P}(S) = \A({\P}(S))$ by minimizing $\AR_{{\P}(S)}(\theta)$.
The robust generalization gap of $\A_{\P}(S)$ over the poisoned distribution ${\P}_\#\Dd$
is called the {\em poisoned generalization gap}, denoted by 
$\varepsilon_{\P}$.
That is,
\begin{align}
    |\E_{S,\A_{\P}}[\AR_{{\P}_\# \Dd}(\A_{\P}(S))- \AR_{{\P}(S)}(\A_{\P}(S))]| &\leq \varepsilon_{\P}.\label{equ:poisoned-generalization-gap-varepsilon-P}
\end{align}

\textbf{Influence of poisoning.}
Our data-dependent bounds in Equ. \eqref{equ:main-thm-convex} and \eqref{equ:main-thm-non-convex} embody the influence of poisoning.
When the distribution $\Dd$ is poisoned by ${\P}$, the bound $\varepsilon(\Dd, \theta_1)$ becomes $\varepsilon(\P_\#\Dd, \theta_1)$.
The expected  curvature at the initialization point becomes $\E_{z}[||\nabla^2 h(\theta_1, {\P}(z))||]$.
The initial population risk gap become $\AR_{\P_\#\Dd}(\theta_1)-\AR_{\P_\#\Dd}(\theta_\P^*)$, in which $\theta_\P^*$ is optimal with respect to $\AR_{\P_\#\Dd}$.
Besides, the variance $\sigma$ also depends on the poisoning and becomes $\sigma_P$.
Additionally, the adversarial population risk $\E_{S,\A_{\P}}[\AR_{{\P}_\#\Dd}(\A_{\P}(S))]$ in
the poisoned counterparts of Equ. \eqref{equ:main-thm-non-convex} and Equ.\eqref{equ:multipass-main-result-non-convex} can be significantly influenced by the poisoning.

\section{Experiments}\label{sec:experiments}
In this section, experiments are used
to demonstrate the data-dependent stability of adversarial training and the advantages of our theoretical results.
We adopt $L_\infty$ norm as constraints of imperceptible perturbations. 
Experimental setups and details are presented in Appendix \ref{app:additional-experiments}.

\begin{figure}[ht]
\begin{center}
\subfigure
{
    \begin{minipage}{0.45\columnwidth}
    \centering
    \includegraphics[width=\columnwidth]{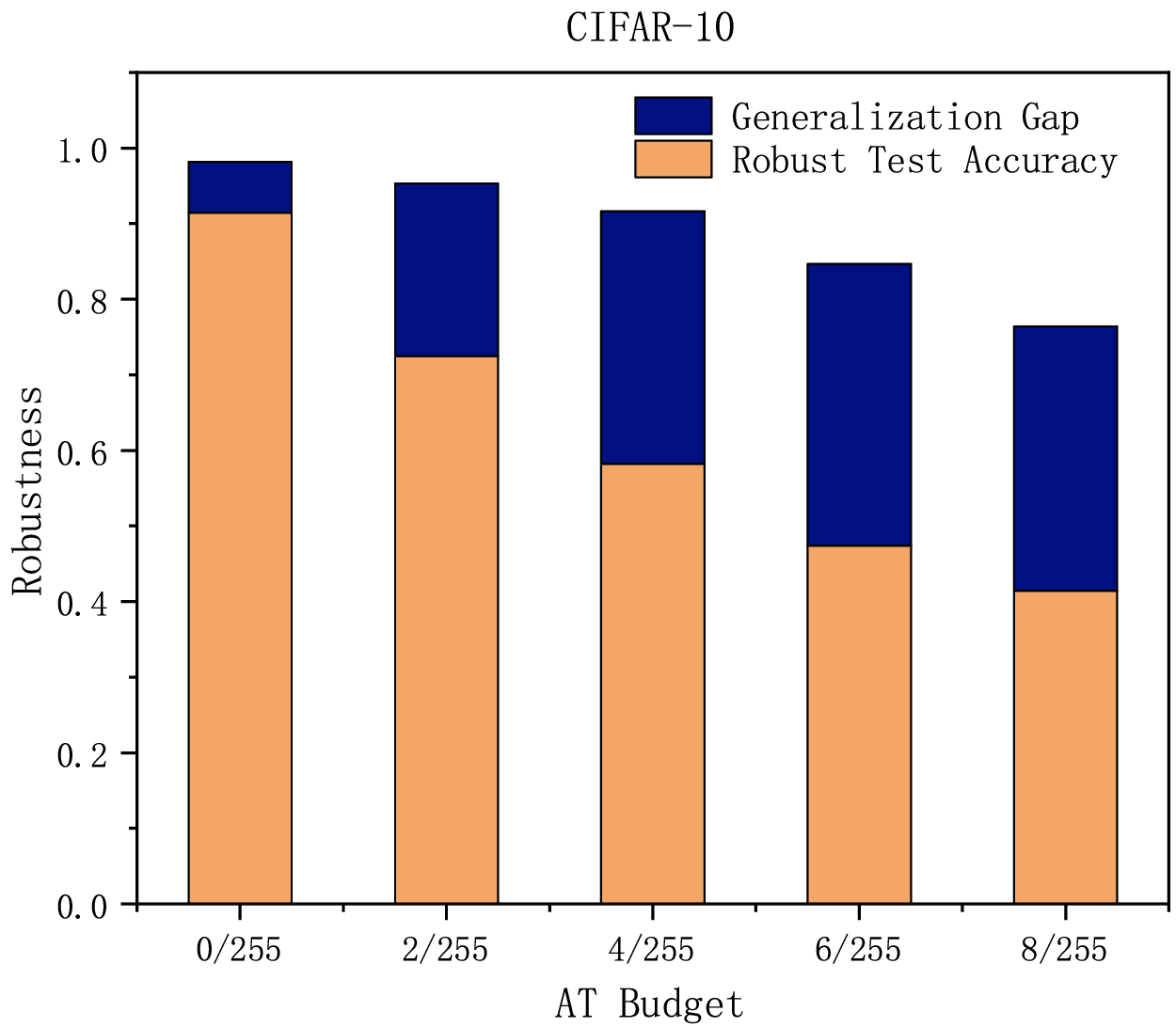}
    \centerline{(a) CIFAR-10}
    \end{minipage}
    \label{subfig:cifar10-budget}
    \hfill
}\subfigure
{
    \begin{minipage}{0.45\columnwidth}
    \centering
    \includegraphics[width=\columnwidth]{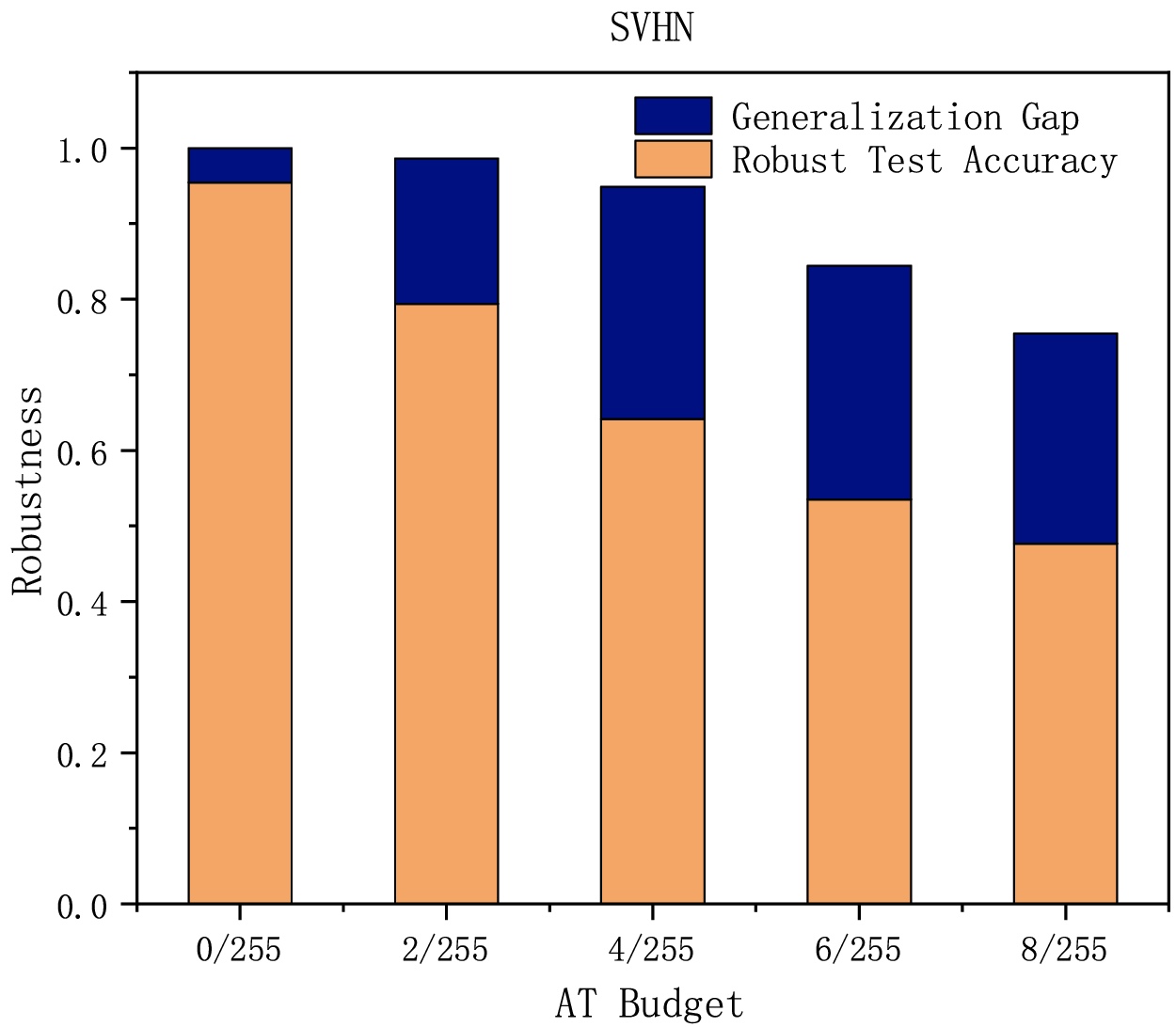}
    \centering{(b) SVHN}
    \end{minipage}
    \label{subfig:svhn-budget}
    \hfill
}
\caption{The robust performance of adversarial training with different with the AT budget $\epsilon$ ranging from $0$ to $8/255$.
}
\end{center}
\end{figure}

\subsection{Robust Generalization.}
We adversarially train  ResNet-18  \cite{he2016deep} 
on CIFAR-10, CIFAR-100 \cite{krizhevsky2009learning},
SVHN \cite{netzer2011reading},
and Tiny-ImageNet \cite{le2015tiny}. 
Figures \ref{subfig:cifar10-budget} 
and \ref{subfig:svhn-budget} show 
that the robust generalization is more difficult than the standard generalization, i.e. $\epsilon=0$ 
as shown by Equ. \eqref{equ:main-result-convex} and Equ. \eqref{equ:main-thm-non-convex}.
The effect of even a small $\epsilon$ such as $2/255$ is amplified by the gradient and Hessian Lipschitz constants in $z$, namely $L_z$ and $H_z$, and results in a large generalization gap. 
Moreover, the robust generalization gap increases with $\epsilon$ which implies that it is harder to ensure robustness in a broader area.
Figures \ref{subfig:cifar100-overfit} 
and \ref{subfig:tiny-imagenet-overfit}
present the robust overfitting phenomenon on CIFAR-100 and Tiny-ImageNet. 
When training errors converge to zero,
the robust generalization gaps (blue lines) grow throughout the whole training procedure, 
while the robust test accuracy (red lines) increases in the first 100 epochs, decreases from the first learning rate decay at the 100-th epoch,
and then jumps a little at the 150-th epoch before stabilizes.

\begin{figure}[ht]
\begin{center}
\subfigure
{
    \begin{minipage}{0.45\columnwidth}
    \centering
    \includegraphics[width=\columnwidth]{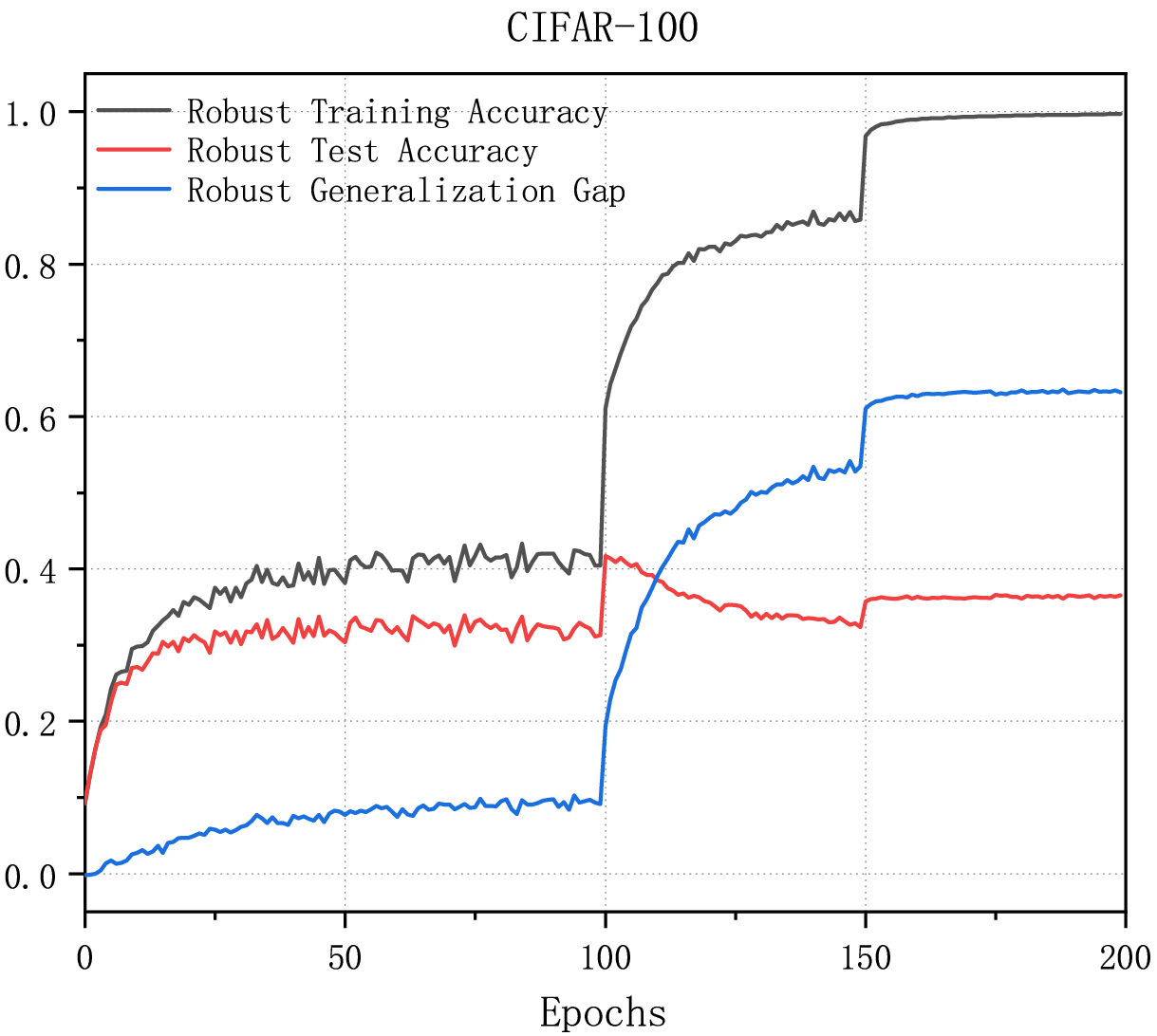}
    \centerline{(c) CIFAR-100}
    \end{minipage}
    \label{subfig:cifar100-overfit}
    \hfill
}\subfigure
{
    \begin{minipage}{0.45\columnwidth}
    \centering
    \includegraphics[width=\columnwidth]{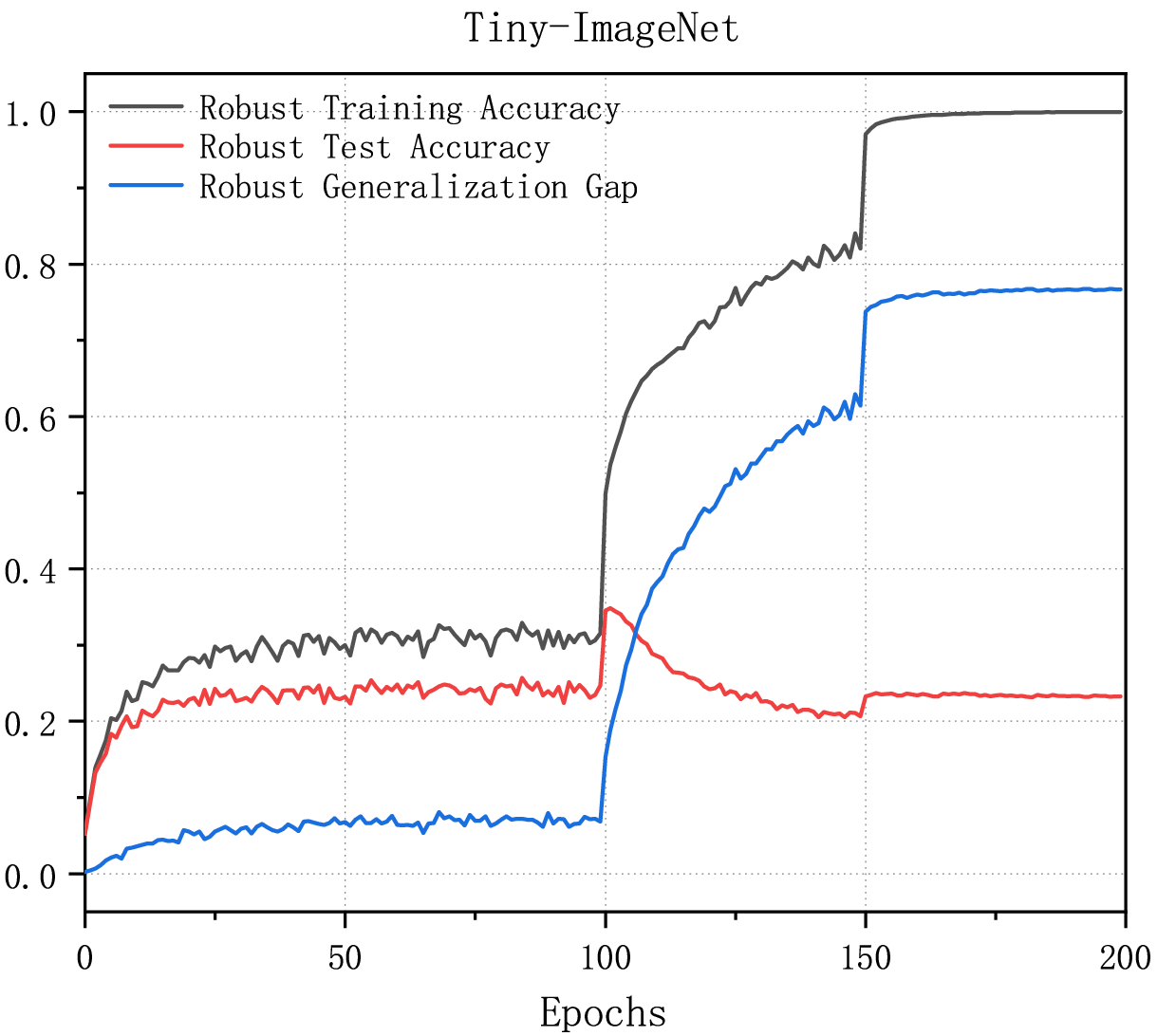}
    \centering{(d) Tiny-ImageNet}
    \end{minipage}
    \label{subfig:tiny-imagenet-overfit}
}
\caption{The robust overfitting phenomenon.
}
\label{fig:cifar10-at-budget}
\end{center}
\end{figure}
\subsection{Poisoned Robust Generalization.}
A poisoning attack is called a {\em stability attack} if the attack aims at destroying the robustness of a model, trained on the poisoned training set ${\P}(S)\sim {\P}_\#\Dd ^n$, on the original distribution $\Dd$, 
i.e. $\AR_{\Dd}(\A_{\P}(S))$.
Stability attacks employed in this paper
include the error-minimizing noise (EM)  \cite{huang2021unlearnable}, 
the robust error-minimizing noise (REM)  \cite{fu2022robust}, 
the adversarial poisoning (ADV)   \cite{fowl2021adversarial}, 
the hypocritical perturbation (HYP)  \cite{tao2022can}
and the class-wise random noise (RAN).
We poison both the training and test sets to simulate the poisoned distribution ${\P}_\#\Dd$.
Detailed poisoning settings are given in Appendix \ref{app:additional-experiments}.

\begin{figure}[ht]
\centering
\subfigure
{
    \begin{minipage}{0.45\columnwidth}
    \centering
    \includegraphics[width=1\columnwidth]{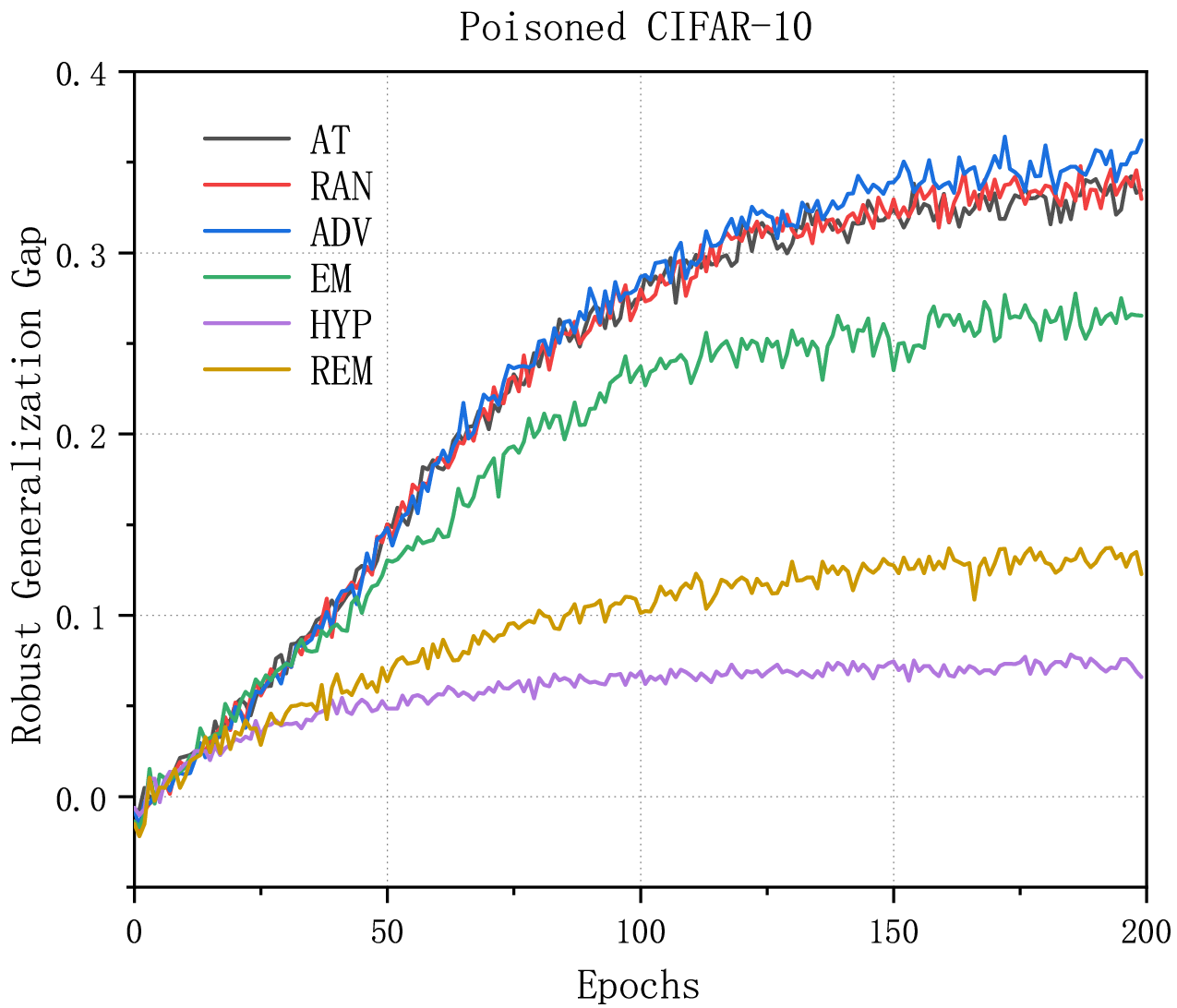}
    \centerline{(a) Poisoned Gen Gap}
    \end{minipage}
    \label{subfig:cifar10-poisoned-gen-gap}
}\subfigure
{
    \begin{minipage}{0.45\columnwidth}
    \centering
    \includegraphics[width=\columnwidth]{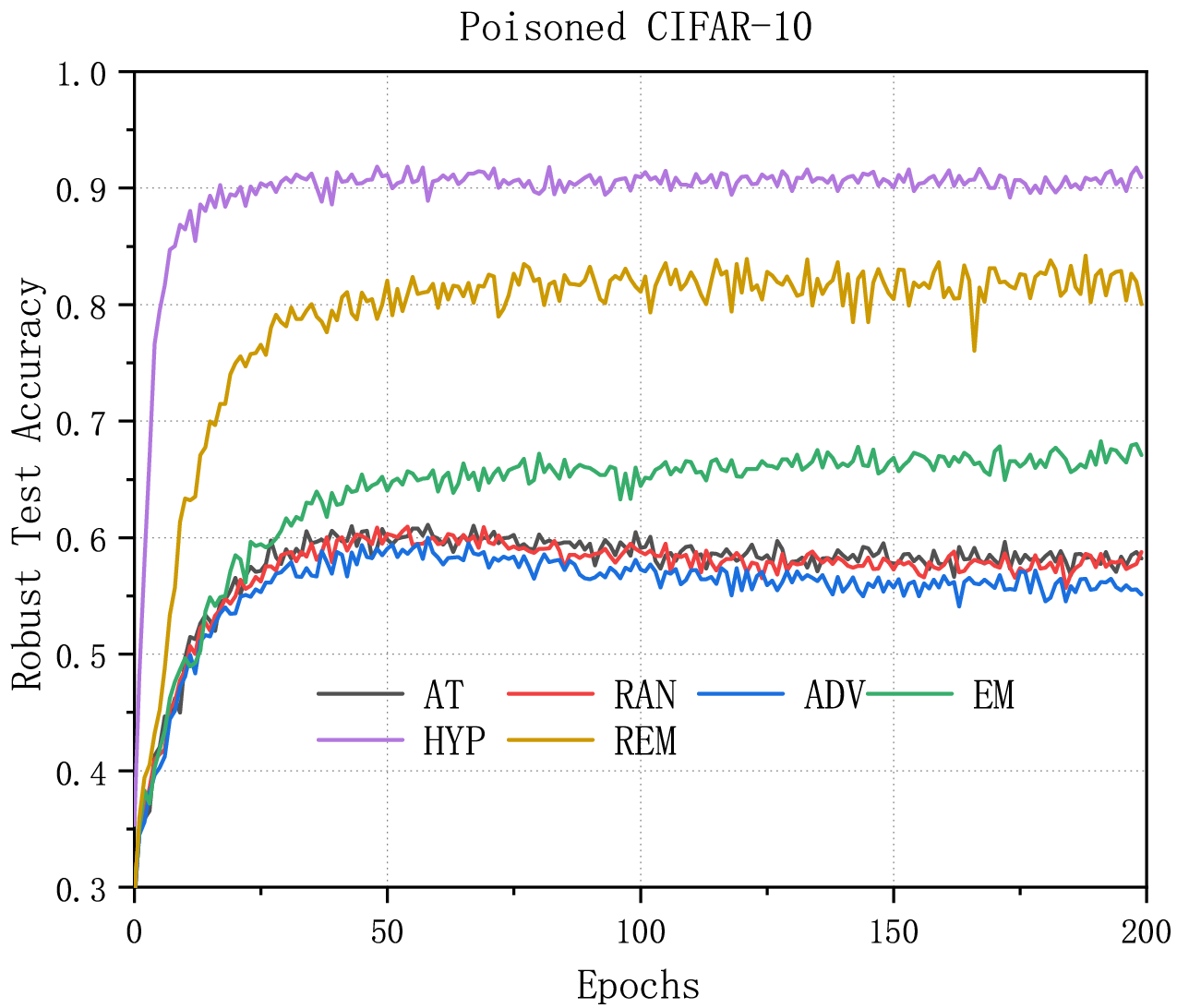}
    \centerline{(b) Poisoned Test Acc}
    \end{minipage}
    \label{subfig:cifar10-test-rob-acc}
}
\caption{
The robust generalization and robust test accuracy
on poisoned CIFAR-10 
under different stability attacks.
The adversarial training budget $\epsilon=4/255$ and the poisoning budget $\epsilon'=8/255$.
}
\end{figure}
Our bounds reflect the influence of data poisoning on the poisoned robust generalization.
First, effective stability attacks such as EM, HYP, and REM, indeed result in the shrinkage of robust generalization gaps on CIFAR-10 and ResNet-18 in Figures \ref{subfig:cifar10-poisoned-gen-gap}.
Comparing Figure \ref{subfig:cifar10-poisoned-gen-gap}
and Figure \ref{subfig:cifar10-test-rob-acc},
we see that robust generalization gaps present correlated trends to the test performance as pointed out by our results, i.e. Equ. \eqref{equ:main-thm-non-convex} and \eqref{equ:multipass-main-result-non-convex}.
We further study the robust generalization under the HYP attack with various intensities, i.e. the poisoning budget $\epsilon'$,
on CIFAR-100.
A larger budget leads to a stronger stability attack.
Figure \ref{subfig:cifar100-hyp-gen} and \ref{subfig:cifar100-hyp-test} show that 
a stronger stability attack results in a 
lower robust test accuracy as well as a narrower robust generalization gap on the poisoned data distribution, 
which confirms the principle stated by our results again.

\begin{figure}[ht]
\centering
\subfigure
{
    \begin{minipage}{0.45\columnwidth}
    \centering
    \includegraphics[width=\columnwidth]{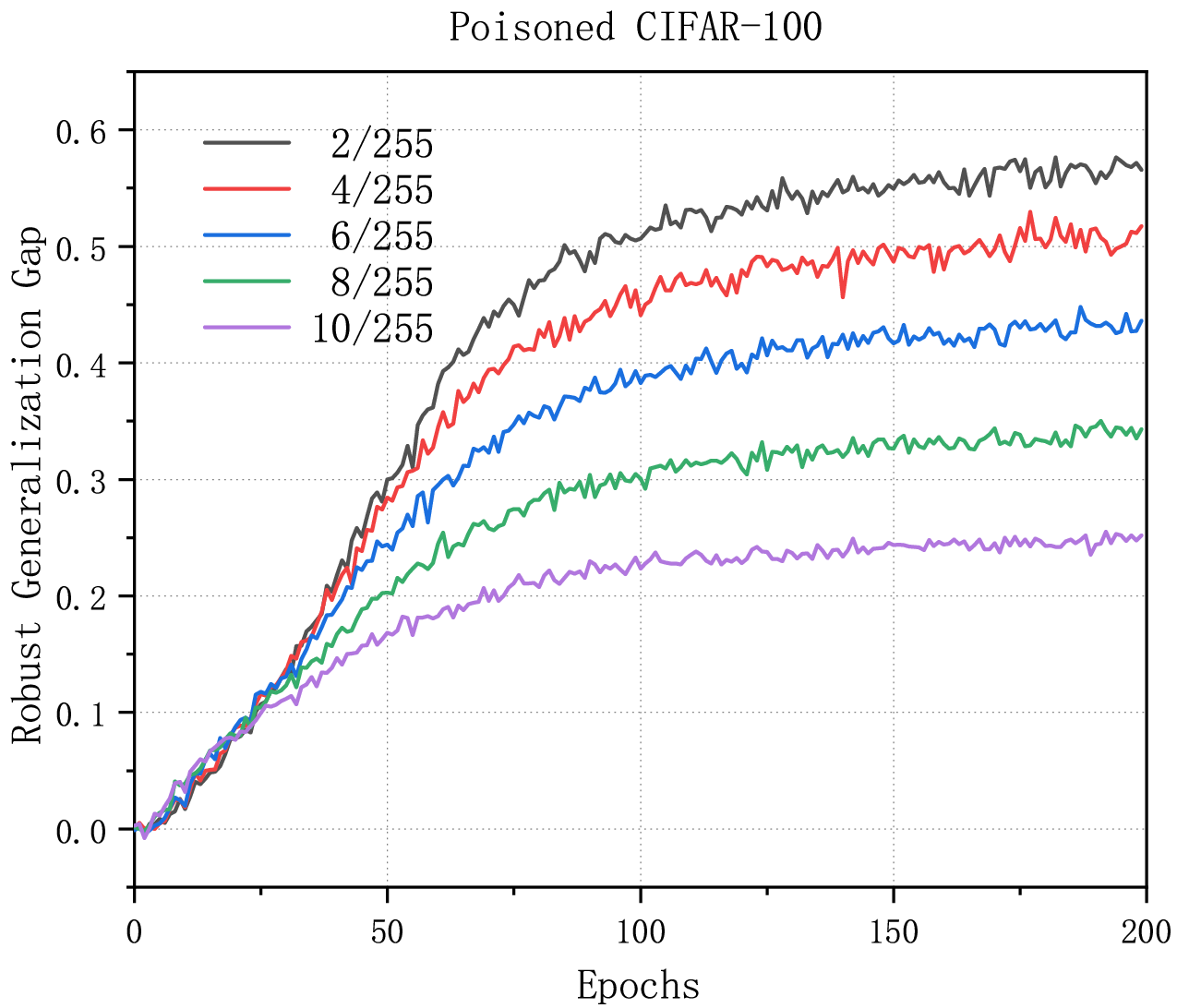}
    \centerline{(a) Poisoned Gen Gap}
    \label{subfig:cifar100-hyp-gen}
    \end{minipage}
}\subfigure
{
    \begin{minipage}{0.45\columnwidth}
\centering
    \includegraphics[width=\columnwidth]{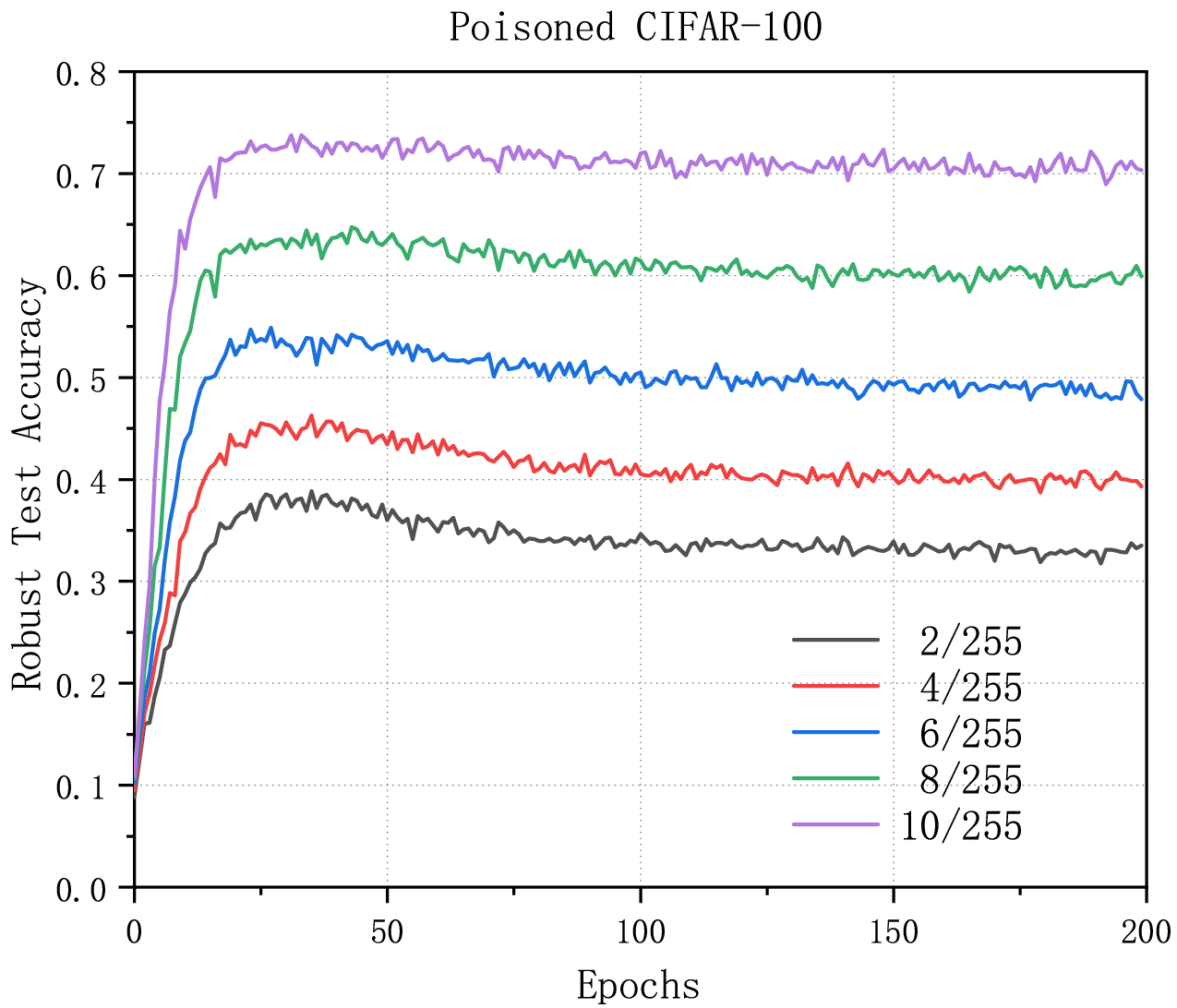}
    \centerline{(a) Poisoned Test Acc}
    \label{subfig:cifar100-hyp-test}
    \end{minipage}
}
\caption{
The robust generalization and robust test accuracy
on the poisoned data 
under HYP attack with different poisoning budgets.
The adversarial training budget $\epsilon=4/255$ and the poisoning budget $\epsilon'$ varies.
}
\end{figure}


\section{Conclusion}

Motivated by the need to analyze the generalization ability for adversarial training under data poisoning attacks, we present a data-dependent stability analysis of adversarial training.
Precisely, under certain reasonable smoothness conditions on the loss functions, we prove that SGD-based adversarial training is an
$\varepsilon(\Dd, \theta_1)$-on-average stable randomized algorithm,
and thus give an upper bound $\varepsilon(\Dd, \theta_1)$ for the robust generalization gap of the training algorithm.
The bound $\varepsilon(\Dd, \theta_1)$ depends on the data distribution and the initial point of the algorithm
and can be used to explain the changes in the poisoned robust generalization gaps of adversarial training.

\paragraph{Limitations and future works}
%
Our theoretical results provide the first attempt to analyze the influence of distribution shifts on robust generalization bounds, but only partial solutions are given.
More refined generalization bounds for adversarial training to capture more relationships between robust generalization and distribution are a future research problem.
%
%
Furthermore, alternative forms of Assumptions \ref{ass:l-gradient-lip} and \ref{ass:l-lip-hessian} for ReLU-based networks need to be further studied.

\bibliography{main}

\newpage
\appendix
\onecolumn
%

\section{Proofs}\label{app:proofs}

\subsection{Proof of Lemma \ref{lem:h-lip-conditions}}
\begin{proof}
\begin{enumerate}
    \item Assume that $h(\theta_1,z)=l(\theta_1,z_1)$ and $h(\theta_2,z)=l(\theta_2,z_2)$. 
    We have 
    \begin{align*}
        ||h(\theta_1,z)-h(\theta_2,z)||=||l(\theta_1,z_1)-l(\theta_2,z_2)||.
    \end{align*}
    Note that $l(\theta_1, z_1)\geq l(\theta_1, z_2)$ and $l(\theta_2, z_2)\geq l(\theta_2, z_1)$.\\
    If $l(\theta_1, z_1)\geq l(\theta_2, z_2)$,
    then
    \begin{align*}
        &||l(\theta_1, z_1) - l(\theta_2,z_2)||
        \leq l(\theta_1, z_1) - l(\theta_2,z_1)
        \leq L||\theta_1-\theta_2||.
    \end{align*}
    If $l(\theta_1, z_1)\leq l(\theta_2, z_2)$,
    then
    \begin{align*}
        &||l(\theta_1, z_1) - l(\theta_2,z_2)||
        \leq l(\theta_2, z_2) - l(\theta_1,z_2)\leq L||\theta_1-\theta_2||.
    \end{align*}
    
    \item Assume that $h(\theta_1,z)=l(\theta_1,z_1)$ and $h(\theta_2,z)=l(\theta_2,z_2)$. 
    \begin{align*}
        &||\nabla h(\theta_1, z)- \nabla h(\theta_2, z)||\\
        =& ||\nabla l(\theta_1, z_1)- \nabla l(\theta_2, z_2)||\\
        \leq& ||\nabla l(\theta_1, z_1)- \nabla l(\theta_1, z_2)||
       +||\nabla l(\theta_1, z_2)- \nabla l(\theta_2, z_2)||\\
        \leq& L_\theta ||\theta_1 - \theta_2|| + L_z ||z_1-z_2||_p\\
        \leq& L_\theta ||\theta_1 - \theta_2|| + 2\epsilon L_z.
    \end{align*}
    
    \item Assume that $h(\theta_1,z)=l(\theta_1,z_1)$ and $h(\theta_2,z)=l(\theta_2,z_2)$. 
    \begin{align*}
        &||\nabla^2 h(\theta_1, z)- \nabla^2 h(\theta_2, z)||\\
        =& ||\nabla^2 l(\theta_1, z_1)- \nabla^2 l(\theta_2, z_2)||\\
        \leq& ||\nabla^2 l(\theta_1, z_1)- \nabla^2 l(\theta_1, z_2)||
        +||\nabla^2 l(\theta_1, z_2)- \nabla^2 l(\theta_2, z_2)||\\
        \leq& H_\theta ||\theta_1 - \theta_2|| + H_z ||z_1-z_2||_p\\
        \leq& H_\theta ||\theta_1 - \theta_2|| + 2\epsilon H_z.
    \end{align*}
\end{enumerate}
\end{proof}

\subsection{Proof of Theorem \ref{thm:convex-main}}
We first prove several lemmas.
A core technique in stability analysis is to give the expansion properties of update rules.
\begin{definition}[Expansion] \label{def:expansion}
    The update rule $\G_\A$ is $\iota$-approximately $\kappa$-expansive, if $\forall z\in \D$
    \begin{align*}
        ||\G_\A(\theta_1, z, \alpha) - \G_\A(\theta_2, z, \alpha)|| \leq \kappa ||\theta_1-\theta_2||+\iota.
    \end{align*}
\end{definition}
If the original loss $l(\theta, z)$ is $\beta$-gradient Lipschitz in $\theta$, then the update rule in standard training is $1$-expansive in the convex case and $(1+\alpha \beta)$-expansive in the non-convex case \cite{hardt2016train}.
If the adversarial loss $h(\theta, z)$ is $\eta$-approximately $\beta$-gradient Lipschitz in $\theta$,
then the expansion coefficients in the update rule $\G_\A$ remain unchanged in both convex and non-convex cases, 
while the approximation parameter $\eta$ leads to an additional term $\alpha \eta$ in each update \cite{xiao2022stability}.

\begin{lemma}[  \citet{xiao2022stability}] \label{lem:ref-xiao-approx-grad-lip}
  Suppose the adversarial loss $h(\theta,z)$ is $\eta$-approximately $\beta$-gradient Lipschitz in $\theta$.
    \begin{enumerate}

        \item ($\eta$-approximate descent.)
            \begin{align*}
                &h(\theta_1, z) - h(\theta_2, z) 
                \leq \nabla h(\theta_2, z)^\top(\theta_1 - \theta_2) 
                + \frac{\beta}{2}||\theta_1 - \theta_2||^2 + \eta ||\theta_1-\theta_2||.  
            \end{align*}

        \item The update rule $\G_\A$ is $\eta$-approximately $(1+\alpha \beta)$-expansive:
            \begin{align*}
                &||\G_\A(\theta_1, z, \alpha)-\G_\A(\theta_2, z, \alpha)||
                \leq (1+\alpha \beta)||\theta_1-\theta_2||+\alpha \eta.
            \end{align*}
        
        \item Assume in addition that $h(\theta,z)$ is convex in $\theta$, for $\alpha \leq 1/\beta$, we have
            \begin{align*}
                ||\G_\A(\theta_1, z, \alpha)-\G_\A(\theta_2, z, \alpha)||\leq ||\theta_1-\theta_2||+\alpha \eta.
            \end{align*}
    \end{enumerate}
\end{lemma}


Given a data set $S=\{z_1,\cdots,z_n\}\sim \Dd^n$, 
an example $z\sim \Dd$, and an index $i\in [n]$, 
we denote $S^{i,z} = \{z'_1,\cdots,z'_n\}$ with $z'_j=z_j$ for $j\neq i$ and $z'_i=z$.
Let $\theta_t$, $\theta_t'$ be the $t$-th outputs of $\A(S)$ and $\A(S^{i,z})$ respectively.
Denote the distance of two trajectories at step $t$ by 
$\delta_t(S, z, i, \A) = ||\theta_t-\theta_t'||$.
As both two updates start from $\theta_1$, we have $\delta_1(S,z, i,\A)=0$.
Since the on-average stability in Definition \ref{def:on-average-stability} takes supremum over the index $i\in [n]$,
the stability analysis aims at providing a unified bound for all $i\in [n]$.
Thus, we will not point out the selection of $i$ in later statements for brevity.

We restate  Lemma 5 in \cite{kuzborskij2018data} on which the data-dependent stability analysis relies.
Note that this lemma holds for SGD without replacement in both a single pass and multiple passes through the training set.
The multiple-pass case cycles through $S$ repeatedly in a fixed order determined by $\A$. 

\begin{lemma}\label{lem:kuz-average-stable-delta}
    Assume the adversarial loss $h(\theta,z)$ is
    non-negative and $L$-Lipschitz in $\theta$. 
    Then, $\forall t_0\in [n+1]$,
    \begin{align*}
        &\E_{S, z,\A}[h(\theta_t, z)-h(\theta_t', z)]
        \leq L \E_{S,z}[\E_\A[\delta_t(S,z,i,\A)|\delta_{t_0}(S,z,i,\A)=0]]
        +\frac{t_0-1}{n}\E_{S,\A}[\AR_\Dd(\theta_t)].
    \end{align*}
\end{lemma}

Due to the change of notations, we repeat the proof here.
\begin{proof}
    By the Lipschitz condition and non-negativeness of $h$, we have
    \begin{align*}
        &h(\theta_t,z) - h(\theta'_t,z)\\
        =&  (h(\theta_t,z) - h(\theta'_t,z))\I\{\delta_{t_0}(S,z,i,\A)=0\} 
        + (h(\theta_t,z) - h(\theta'_t,z)) \I\{\delta_{t_0}(S,z,i,\A)\neq 0\}\\
        \leq& L\delta_{t}(S,z,i,\A)\I\{\delta_{t_0}(S,z,i,\A)=0\} 
        + h(\theta_t,z)\I\{\delta_{t_0}(S,z,i,\A)\neq 0\}.
    \end{align*}
    Take expectation w.r.t. $\A$ and we have
    \begin{align}
        &\E_\A[h(\theta_t,z) - h(\theta'_t,z)]
        \leq L \E_\A[\delta_{t}(S,z,i,\A)|\delta_{t_0}(S,z,i,\A)=0] 
        + \E_\A[ h(\theta_t,z)\I\{\delta_{t_0}(S,z,i,\A)\neq 0\}].\label{equ:kuz-lem-equ1}
    \end{align}
    Note that the first time that $\A$ selects the different example is $\pi^{-1}(i)$.
    Since that $\pi^{-1}(i) \geq t_0$ implies $\delta_{t_0}(S,z,i,\A)=0$,
    we have $\I\{\delta_{t_0}(S,z,i,\A)\neq 0\}\leq \I\{\pi^{-1}(i)< t_0 \}$.
    It follows that 
    \begin{align}
        &\E_{S,z}[\E_\A[ h(\theta_t,z)\I\{\delta_{t_0}(S,z,i,\A)\neq 0\}] ]]\nonumber\\
        &\leq  \E_{S,z}[\E_\A[ h(\theta_t,z)\I\{\pi^{-1}(i)< t_0\}]]\nonumber\\
        &= \E_{z, \A}[ \E_{S}[h(\theta_t,z)]\I\{\pi^{-1}(i)< t_0\}].\label{equ:kuz-lem-equ2}
    \end{align}
    Recall that a realization of $\A$ is a permutation $\pi$ of $[n]$.
    Thus, with a fixed $\pi$, taking over $S\sim \Dd^n$ equals to taking over both $S\sim \Dd^n$ and $\A$.
    That is, $\E_S[h(\theta_t,z)] = \E_{\A, S}[h(\theta_t,z)]$.
    As a consequence, we have
    \begin{align}
        &\E_{z,\A}[ \E_{S}[h(\theta_t,z)]\I\{\pi^{-1}(i)< t_0\}]\\
        =& \E_{S, z,\A}[h(\theta_t,z)]\E_{\A}[\I\{\pi^{-1}(i)< t_0\}]\nonumber\\
        \leq& \frac{t_0-1}{n}\E_{S, z,\A}[h(\theta_t,z)]. \label{equ:kuz-lem-equ3}
    \end{align}
    Combining Equation~\eqref{equ:kuz-lem-equ1} ~\eqref{equ:kuz-lem-equ2} and ~\eqref{equ:kuz-lem-equ3}, we get the statement.
\end{proof}


\begin{lemma} \label{lem:sum-R-inequality}
    Suppose the adversarial loss $h(\theta, z)$ is $L$-Lipschitz and $\eta$-approximately $\beta$-gradient Lipschitz in $\theta$.
    Then, in a single pass such that $T\in [n]$, we have that
    \begin{align*}
        &\sum_{t=1}^{T} (\alpha_t - \frac{\beta \alpha_t^2}{2})\E_{S}[||\nabla \AR_\Dd(\theta_t)||^2]
        \leq  \AR_\Dd(\theta_1) - \E_{S}[\AR_\Dd(\theta_T)] 
        + \eta L \sum_{t=1}^{T}\alpha_t 
        + \frac{\beta}{2}\sum_{t=1}^{T}\alpha_t^2 
            \E_{S}[||\nabla h(\theta_t,z_{\pi(t)}) - \nabla \AR_\Dd(\theta_t)||^2].
    \end{align*}
\end{lemma}

\begin{proof}
    From the first statement in Lemma \ref{lem:ref-xiao-approx-grad-lip}, we have
    \begin{align*}
        &\AR_\Dd(\theta_{t+1}) - \AR_\Dd(\theta_t)\\ 
        \leq& \nabla \AR_\Dd(\theta_t)^\top(\theta_{t+1}-\theta_t) +\frac{\beta\alpha_t^2}{2}||\nabla h(\theta_t, z_{\pi(t)})||^2 +\eta \alpha_t||\nabla h(\theta_t, z_{\pi(t)})||\\
        \leq& (\beta \alpha_t^2 - \alpha_t)\nabla \AR_\Dd(\theta_t)^\top \nabla h(\theta_t, z_{\pi(t)}) + \eta \alpha_t L
        + \frac{\beta \alpha_t^2}{2} ||\nabla  h(\theta_t, z_{\pi(t)}) - \nabla \AR_\Dd(\theta_t)||^2 - \frac{\beta\alpha_t^2}{2}||\nabla \AR_\Dd(\theta_t)||^2 .
    \end{align*}
    
    Since $\theta_t$ is determined by $z_{\pi(1)},\cdots,z_{\pi(t-1)}$ and $\E_{z_{\pi(t)}}[h(\theta_t, z_{\pi(t)})] = \AR_\Dd(\theta_t)$,
    we have that 
    \begin{align*}
        \E_{S}[ \nabla \AR_\Dd(\theta_t)^\top\nabla h(\theta_t, z_{\pi(t)})] = \E_{S}[ ||\nabla \AR_\Dd(\theta_t)||^2].
    \end{align*}
    
    Take expectation w.r.t. $S$ and rearrange terms,
    \begin{align*}
        &(\alpha_t - \frac{\beta \alpha_t^2}{2})\E_{S}[ ||\nabla \AR_\Dd(\theta_t)||^2]
        \leq \E_{S}[\AR_\Dd(\theta_{t}) - \AR_\Dd(\theta_{t+1})
        +\frac{\beta \alpha_t^2}{2} ||\nabla  h(\theta_t, z_{\pi(t)}) - \nabla \AR_\Dd(\theta_t)||^2] + \eta \alpha_t L.
    \end{align*}
    Sum the above over $t=1,\cdots, T$ and get the statement.
\end{proof}


\begin{lemma} \label{lem:sum-grad-h-inequality}
    Suppose the adversarial loss $h(\theta,z)$ is $L$-Lipschitz and $\eta$-approximately $\beta$-gradient Lipschitz with respect to $\theta$, 
    and the step sizes $\alpha_t\leq 1/\beta$.
    Assume the variance of stochastic gradients in $\A$ obeys for all $t\in[T]$
    \begin{align*}
        \E_{S}[||\nabla h(\theta_t, z_{\pi(t)}) - \nabla \AR_\Dd(\theta_t)||^2] \leq \sigma_t^2.
    \end{align*}
    We have 
    \begin{align*}
        &\E_{S}[\sum_{t=1}^{T}\alpha_t ||\nabla h(\theta_t, z_{\pi(t)})|| ]
        \leq \sum_{t=1}^{T}\sigma_t \alpha_t 
        +  2\sqrt{\sum_{t=1}^{T}\alpha_t} \sqrt{\AR_\Dd(\theta_1)-\AR_\Dd(\theta^*) + \frac{\beta}{2} \sum_{t=1}^{T} \sigma_t^2\alpha_t^2 + \eta L \sum_{t=1}^{T}\alpha_t}
    \end{align*}
\end{lemma}

\begin{proof}
    Repeatedly applying Jensen's inequality,
    we have
    \begin{align*}
        &\E_{S}[\sum_{t=1}^{T}\alpha_t ||\nabla h(\theta_t, z_{\pi(t)})|| ]\\
        \leq& \sum_{t=1}^{T}\alpha_t \E_{S}[ ||\nabla h(\theta_t, z_{\pi(t)})-\nabla \AR_\Dd(\theta_t)|| ]
        +\sum_{t=1}^{T}\alpha_t \E_{S}[ ||\nabla \AR_\Dd(\theta_t)|| ]\\
        \leq& \sum_{t=1}^{T}\alpha_t \sqrt{\E_{S}[ ||\nabla h(\theta_t, z_{\pi(t)})-\nabla \AR_\Dd(\theta_t)||^2 ]}
        +\sum_{t=1}^{T}\alpha_t \sqrt{\E_{S}[ ||\nabla \AR_\Dd(\theta_t)||^2 ]}\\
        \leq&  \sum_{t=1}^{T}\sigma_t \alpha_t +\sum_{t=1}^{T}\alpha_t \sqrt{\E_{S}[ ||\nabla \AR_\Dd(\theta_t)||^2 ]}\\
        \leq&  \sum_{t=1}^{T}\sigma_t\alpha_t + 2\sum_{t=1}^{T}(\alpha_t-\frac{\beta\alpha_t^2}{2}) \sqrt{\E_{S}[ ||\nabla \AR_\Dd(\theta_t)||^2 ]}\\
        \leq&  \sum_{t=1}^{T}\sigma_t\alpha_t 
        +  2\sqrt{\sum_{t=1}^{T}(\alpha_t-\frac{\beta\alpha_t^2}{2})} \sqrt{\sum_{t=1}^{T}(\alpha_t-\frac{\beta\alpha_t^2}{2}) \E_{S}[ ||\nabla \AR_\Dd(\theta_t)||^2 ]}\\
        \leq&  \sum_{t=1}^{T}\sigma_t \alpha_t 
        +  2\sqrt{\sum_{t=1}^{T}\alpha_t} \sqrt{\AR_\Dd(\theta_1)-\AR_\Dd(\theta^*) + \frac{\beta}{2} \sum_{t=1}^{T} \sigma_t^2\alpha_t^2 + \eta L \sum_{t=1}^{T}\alpha_t}.
    \end{align*}
    The penultimate inequality is by Lemma \ref{lem:sum-R-inequality}.
\end{proof}

We now prove  Theorem \ref{thm:convex-main}.
\begin{proof}
    Denote $\Delta_t(S,z,i)=\E_\A[\delta_t(S,z,i, \A)|\delta_{t_0}(S,z,i,\A)=0]$.
    By Lemma \ref{lem:kuz-average-stable-delta}, $\forall t_0 \in \{1,\cdots,n, n+1\}$ we have
    \begin{align*}
        &\E_{S, z,\A}[h(\theta_{T+1}, z)-h(\theta_{T+1}', z)]\\
        \leq& L \E_{S,z}[\Delta_{T+1}(S,z,i)]+\frac{t_0-1}{n}\E_{S,\A}[\AR_\Dd(\theta_{T+1})].
    \end{align*}
    At  step $t$, $\A$ selects the example $\pi(t)=i$ with probability $1/n$ and 
    $\pi(t)\neq i$ with probability $1-1/n$.
    When $\pi(t)\neq i$, by the third statement in Lemma \ref{lem:ref-xiao-approx-grad-lip} ,
    we have
    \begin{align*}
        &\delta_{t+1}(S,z,i,\A)\cdot \I\{\delta_{t_0}(S,z,i,\A)=0\}\\
        \leq &\delta_t(S,z,i,\A)\cdot \I\{\delta_{t_0}(S,z,i,\A)=0\} + \alpha_t\eta.
    \end{align*}
    When $\pi(t)=i$, we have
    \begin{align*}
        &\delta_{t+1}(S,z,i,\A)\cdot \I\{\delta_{t_0}(S,z,i,\A)=0\}\\
        \leq& \delta_t(S,z,i,\A)\cdot \I\{\delta_{t_0}(S,z,i,\A)=0\} 
        + \alpha_t||\nabla h(\theta_t,z_{\pi(t)})||+\alpha_t||\nabla h(\theta_t',z'_{\pi(t)})||.
    \end{align*}
    Take expectation w.r.t. $\A$ and we have
    \begin{align*}
        &\Delta_{t+1}(S,z,i)\\
        \leq& \frac{1}{n}(\Delta_t(S,z,i)
        + \alpha_t\E_\A[||\nabla h(\theta_t,z_{\pi(t)})||+||\nabla h(\theta_t', z'_{\pi(t)})||])
        + (1-\frac{1}{n})(\Delta_t(S,z,i) + \alpha_t \eta)\\
        =& \Delta_t(S,z,i) + (1-\frac{1}{n})\alpha_t \eta 
        + \frac{\alpha_t}{n}\E_\A[||\nabla h(\theta_t,z_{\pi(t)})||+||\nabla h(\theta_t',z'_{\pi(t)})||].
    \end{align*}
    Thus, we have
    \begin{align*}
         &\E_{S, z,\A}[h(\theta_{T+1}, z)-h(\theta_{T+1}', z)]\\
         \leq& \frac{L}{n}\sum_{t=t_0}^T \alpha_t \E_{z,\A}[\E_{S}[||\nabla h(\theta_t,z_{\pi(t)})||+||\nabla h(\theta_t',z'_{\pi(t)})||]]
         + (1-\frac{1}{n})L\eta \sum_{t=t_0}^T \alpha_t
         +\frac{t_0-1}{n}\E_{S,\A}[\AR_\Dd(\theta_T)].
    \end{align*}
    Here we take $t_0=1$.
    By Lemma \ref{lem:sum-grad-h-inequality},
    we have
    \begin{align*}
        &\E_{S}[\sum_{t=1}^{T}\alpha_t||\nabla h(\theta_t,z_{\pi(t)})||]\\
        \leq&  \sum_{t=1}^{T}\sigma\alpha_t +  2\sqrt{\sum_{t=1}^{T}\alpha_t}
        \cdot \sqrt{\AR_\Dd(\theta_1)-\AR_\Dd(\theta^*) + \frac{\beta}{2} \sum_{t=1}^{T} \sigma^2\alpha_t^2 + \eta L \sum_{t=1}^{T}\alpha_t},
    \end{align*}
    and 
    \begin{align*}
        &\E_{z,S}[\sum_{t=1}^{T}\alpha_t||\nabla h(\theta_t',z'_{\pi(t)})||]\\
        =&\E_{S}[\sum_{t=1}^{T}\alpha_t||\nabla h(\theta_t,z_{\pi(t)})||]\\
        \leq& \sum_{t=1}^{T}\sigma\alpha_t +  2\sqrt{\sum_{t=1}^{T}\alpha_t} 
        \cdot \sqrt{\AR_\Dd(\theta_1)-\AR_\Dd(\theta^*) + \frac{\beta}{2} \sum_{t=1}^{T} \sigma^2\alpha_t^2 + \eta L \sum_{t=1}^{T}\alpha_t},
    \end{align*}
    Thus,
    \begin{align*}
        &\E_{S, z,\A}[h(\theta_{T+1}, z)-h(\theta_{T+1}', z)]\\
        \leq&  \frac{2 L}{n}\sum_{t=1}^{T}\sigma\alpha_t+L\eta\sum_{t=1}^{T}\alpha_t 
        +  \frac{4L}{n}\sqrt{\sum_{t=1}^{T}\alpha_t} 
        \cdot \sqrt{\AR_\Dd(\theta_1)-\AR_\Dd(\theta^*) + \frac{\beta}{2} \sum_{t=1}^{T}\sigma^2\alpha_t^2 + \eta L \sum_{t=1}^{T}\alpha_t}.
    \end{align*}
    
\end{proof}

\subsection{Proof of Corollary \ref{cor:convex-main-constant}}
\begin{proof}
    \begin{align*}
        &(\frac{2\sigma L}{n}+L\eta) \sum_{t=1}^{T}\alpha_t +  \frac{4L}{n}\sqrt{\sum_{t=1}^{T}\alpha_t}
        \cdot \sqrt{\AR_\Dd(\theta_1)-\AR_\Dd(\theta^*) + \frac{\beta \sigma^2}{2} \sum_{t=1}^{T}\alpha_t^2 + \eta L \sum_{t=1}^{T}\alpha_t}\\
        =& (\frac{2\sigma L}{n}+L\eta)\alpha T 
        +\frac{4L}{n}\sqrt{\alpha T(\AR_\Dd(\theta_1)-\AR_\Dd(\theta^*)+\frac{\beta \sigma^2\alpha^3 T^2}{2} + \eta L \alpha^2 T^2)}\\
        \leq& (\frac{2\sigma L}{n}+L\eta)\alpha T +\frac{4L}{n}\sqrt{\alpha T r+\frac{\sigma^2\alpha^2 T^2}{2} + \eta L \alpha^2 T^2}\\
        \leq& \eta\alpha T L+ \frac{2\sigma \alpha T L}{n}+\frac{4L}{n}(\sqrt{\alpha Tr}+\frac{\sigma\alpha T}{\sqrt{2}} + \sqrt{\eta L} \alpha T)\\
        =& \eta\alpha T L+ \frac{2\alpha T L}{n}(\sigma+\sqrt{2}\sigma + 2\sqrt{\eta  L}) + \frac{4L\sqrt{\alpha T r}}{n}.
    \end{align*}
\end{proof}

\subsection{Proof of Theorem \ref{thm:non-convex-main}}
We first prove several lemmas.
\begin{lemma} \label{lem:non-convex-xi}
    Suppose the adversarial loss $h(\theta,z)$ is
    $\nu$-approximately $\rho$-Hessian Lipschitz 
    with respect to $\theta$.
    At step $t$ with $\pi(t)\neq i$, we have
    \begin{align*}
        ||\G_\A(\theta_t) - \G_\A(\theta'_t)||
        \leq (1+\alpha_t \xi_t(S,z,i,\A))\delta_t(S,z,i,\A),
    \end{align*}
    where
    \begin{align*}
        &\xi_t(S,z,i,\A) 
        = ||\nabla h(\theta_1,z_{\pi(t)})|| 
        + \frac{\rho}{2}(\sum_{k=1}^{t-1}\alpha_k(||\nabla h(\theta_k, z_{\pi(k)}||+||\nabla h(\theta'_k, z'_{\pi(k)})||) + \nu.
    \end{align*}
    Furthermore, when $\alpha_k=\frac{c}{k}$ with $c\leq \frac{1}{\beta}$, we have
    \begin{align*}
        &\E_{S,z}[\xi_t(S,z,i,\A)] \\
        \leq& \E_{z}[||\nabla^2 h(\theta_1,z)|| ] + \nu 
        +2\rho\sqrt{\AR_\Dd(\theta_1)-\AR_\Dd(\theta^*)c(1+\ln t)} 
        +2\rho\sigma c\sqrt{\beta c (1+\ln t)} + \rho c(\sigma + 2\sqrt{\eta L})(1+\ln t).
    \end{align*}
\end{lemma}

\begin{proof}
    For $\pi(t)\neq i$, we have
    \begin{align*}
        &||\G_\A(\theta_t) - \G_\A(\theta'_t)||
        \leq ||\theta_t - \theta'_t|| + \alpha_t||\nabla h(\theta_t, z_{\pi(t)}) - \nabla h(\theta'_t, z_{\pi(t)})||.
    \end{align*}
    For brevity, we denote $h_t(\theta) = h(\theta, z_{\pi(t)})$.
    By Taylor expansion with integral remainder,
    we have
    \begin{align*}
        &\nabla h_t(\theta_t) - \nabla h_t(\theta'_t) \\
        =& \int_0 ^1 \nabla^2 h_t(\theta_t + \tau (\theta'_t - \theta_t)) d \tau  \cdot (\theta_t - \theta'_t)\\
        =& \int_0 ^1 (\nabla^2 h_t(\theta_t + \tau (\theta'_t - \theta_t)) - \nabla^2 h_t(\theta_1)) d \tau  \cdot (\theta_t - \theta'_t) 
        + \nabla^2 h_t(\theta_1)  \cdot (\theta_t - \theta'_t).
    \end{align*}
    Since $h$ is $\nu$-approximately $\rho$-Hessian Lipschitz,
    \begin{align*}
        &||\nabla h_t(\theta_t) - \nabla h_t(\theta'_t) || \\
        &\leq (\rho \int_0 ^1 ||\theta_t + \tau (\theta'_t - \theta_t) - \theta_1|| d \tau 
        + \nu +||\nabla^2 h_t(\theta_1)||)\cdot ||\theta_t - \theta'_t||.
    \end{align*}
    Note that
    \begin{align*}
        &\theta_t + \tau (\theta'_t - \theta_t) - \theta_1\\
        &=(1-\tau)(\theta_t-\theta_1) + \tau (\theta'_t - \theta_1)\\
        &=(1-\tau)\sum_{k=1}^{t-1}(\theta_{k+1}-\theta_k) + \tau \sum_{k=1}^{t-1}(\theta'_{k+1}-\theta'_k)\\
        &=(1-\tau)\sum_{k=1}^{t-1}\alpha_k \nabla h(\theta_k, z_{\pi(k)}) +\tau\sum_{k=1}^{t-1}\alpha_k \nabla h_k(\theta'_k,z'_{\pi(k)}).
    \end{align*}
    Therefore,
    \begin{align*}
      &\int_0 ^1 ||\theta_t + \tau (\theta'_t - \theta_t) - \theta_1|| d \tau\\
      &\leq \frac{1}{2}(\sum_{k=1}^{t-1}\alpha_k(||\nabla h(\theta_k, z_{\pi(k)}||+||\nabla h(\theta'_k, z'_{\pi(k)})||).
    \end{align*}
    Taking $\alpha_k = \frac{c}{k}$ and
    by Lemma \ref{lem:sum-grad-h-inequality}, we have
    \begin{align*}
        &\frac{1}{2}(\sum_{k=1}^{t-1}\alpha_k(||\nabla h(\theta_k, z_{\pi(k)}||+||\nabla h(\theta'_k, z'_{\pi(k)}||))\\
        \leq& \sigma \sum_{k=1}^{t-1}\alpha_k +  2\sqrt{\sum_{k=1}^{t-1}\alpha_k} 
        \cdot \sqrt{\AR_\Dd(\theta_1)-\AR_\Dd(\theta^*) + \frac{\beta \sigma^2}{2} \sum_{k=1}^{t-1}\alpha_k^2 + \eta L \sum_{k=1}^{t-1}\alpha_k}\\
        \leq& c\sigma (1+\ln t) +  2\sqrt{c(1+\ln t)}
        \cdot\sqrt{(\AR_\Dd(\theta_1)-\AR_\Dd(\theta^*)+\beta c^2 \sigma^2 + c\eta L(1 + \ln t) }\\
        \leq& 2\sqrt{\AR_\Dd(\theta_1)-\AR_\Dd(\theta^*)c(1+\ln t)} 
        +2\sigma c\sqrt{\beta c (1+\ln t)}
        + c(\sigma + 2\sqrt{\eta L})(1+\ln t).
    \end{align*}
    The penultimate inequality is due to
    \begin{align*}
        \sum_{k=1}^t \frac{1}{k} \leq 1+\ln t,
        \text{    and    }
        \sum_{k=1}^t \frac{1}{k^2} \leq 2 - \frac{1}{t}.
    \end{align*}
    
\end{proof}

\begin{lemma}[Bernstein-type inequality  \cite{kuzborskij2018data}] \label{lem:kuz-bernstein}
Let $Z$ be a zero-mean real-valued random variable such that 
$|Z|\leq b$ and $\E[Z^2]\leq \sigma^2$. 
Then for all $|c|\leq \frac{1}{2b}$,  we have that 
$\E[e^{c Z}]\leq e^{c^2 \sigma^2}$.
\end{lemma}

We now prove Theorem \ref{thm:non-convex-main}.
\begin{proof}
    Let $\Delta_t(S,z,i)=\E_\A[\delta_t(S,z,i,\A)|\delta_{t_0}(S,z,i,\A)=0]$.
    By Lemma \ref{lem:kuz-average-stable-delta}, $\forall t_0\in[n+1]$,
    \begin{align*}
    &\E_{S, z,\A}[h(\theta_{T+1}, z)-h(\theta_{T+1}', z)]
    \leq L \E_{S,z}[\Delta_{T+1}(S,z,i)]+\frac{t_0-1}{n}\E_{S,\A}[\AR_\Dd(\theta_{T+1})].
    \end{align*}
    When $\pi(t)=i$ with probability $\frac{1}{n}$, we have
    \begin{align*}
        ||\G_\A(\theta_t) - \G_\A(\theta'_t)||\leq \delta_t(S,z,i,\A) + 2\alpha_t L.
    \end{align*}
    When $\pi(t)\neq i$ with probability $1-\frac{1}{n}$,
    we have
    \begin{align*}
        ||\G_\A(\theta_t) - \G_\A(\theta'_t)||\leq (1+\alpha_t \beta)\delta_t(S,z,i,\A)+\alpha_t \eta,
    \end{align*}
    by the second statement in Lemma \ref{lem:ref-xiao-approx-grad-lip}
    and 
    \begin{align*}
        ||\G_\A(\theta_t) - \G_\A(\theta'_t)||\leq (1+\alpha_t \xi_t(S,z,i,\A))\delta_t(S,z,i,\A),
    \end{align*}
    by Lemma \ref{lem:non-convex-xi}.
    Let 
    \begin{align*}
        \psi_t(S,z,i) =\E_{\A}[ \min\{\xi_t(S,z,i,\A),\beta \}]
    \end{align*}
    and we have
    \begin{align*}
        &\Delta_{t+1}(S,z,i)\\
        \leq& \frac{1}{n} (\Delta_t(S,z,i)+2\alpha_t L)
        + (1-\frac{1}{n}) ((1+\alpha_t \psi_t(S,z,i))\Delta_t(S,z,i) + \alpha_t \eta)\\
        =& (1+(1-\frac{1}{n})\alpha_t \psi_t(S,z,i))\Delta_t(S,z,i) + \frac{2\alpha_t L + (n-1)\alpha_t \eta}{n}\\
        \leq& \exp((1-\frac{1}{n})\alpha_t \psi_t(S,z,i))\Delta_t(S,z,i) + \frac{2\alpha_t L}{n}+\alpha_t \eta.
    \end{align*}
    Note that $\Delta_{t_0}(S,z,i)=0$ and $\alpha_t = \frac{c}{t}$.  
    We have 
        \begin{align*}
        &\Delta_{T+1}(S,z,i)\\
        \leq& \sum_{t=t_0}^{T}(\prod_{k=t+1}^{T} \exp(\frac{(n-1)c\psi_k(S,z,i)}{nk})) (\frac{2c L}{n t}+\frac{c\eta}{t})\\
        =&\sum_{t=t_0}^{T}\exp(\frac{(n-1)c}{n}\sum_{k=t+1}^{T} \frac{\psi_k(S,z,i)}{k}) (\frac{2c L}{n t}+\frac{c\eta}{t}).
    \end{align*}
    Let $\mu_k = \E_{S,z}[\psi_k(S,z,i)]$.
    We have $|\psi_k(S,z,i)-\mu_k|\leq 2\beta$ and 
    \begin{align*}
        &\E_{S,z}[\exp(c\sum_{k=t+1}^{T} \frac{\psi_k(S,z,i)}{k})]\\
        &=\E_{S,z}[\exp(c\sum_{k=t+1}^{T}\frac{\psi_k(S,z,i)-\mu_k}{k})]\exp(c\sum_{k=t+1}^{T} \frac{\mu_k}{k}).
    \end{align*}
    Since
    \begin{align*}
        |\sum_{k=t+1}^{T} \frac{\psi_k(S,z,i)-\mu_k}{k}|\leq 2 \beta \ln T,
    \end{align*}
    we assume $c \leq \min \{\frac{1}{2(2\beta \ln T)^2}, 
    \frac{1}{2(2\beta \ln T)}\}$.
    By Lemma \ref{lem:kuz-bernstein}, we have
    \begin{align*}
        &\E_{S,z}[\exp(c\sum_{k=t+1}^{T} \frac{\psi_k(S,z,i)-\mu_k}{k})]\\
        &\leq \exp(c^2\E_{S,z}[(\sum_{k=t+1}^{T} \frac{\psi_k(S,z,i)-\mu_k}{k})^2])\\
        &\leq \exp(\frac{c}{2}\E_{S,z}[(\frac{1}{2\beta \ln T}\sum_{k=t+1}^{T} \frac{\psi_k(S,z,i)-\mu_k}{k})^2])\\
        &\leq \exp(\frac{c}{2}\E_{S,z}[|\sum_{k=t+1}^{T} \frac{\psi_k(S,z,i)-\mu_k}{k}|])\\
        &\leq \exp(\frac{c}{2}\sum_{k=t+1}^{T} \frac{\E_{S,z}[|\psi_k(S,z,i)-\mu_k|]}{k})\\
        &\leq \exp( c\sum_{k=t+1}^{T} \frac{\mu_k}{k}).
    \end{align*}
    It follows that 
    \begin{align*}
         \E_{S,z}[\exp(c\sum_{k=t+1}^{T} \frac{\psi_k(S,z,i)}{k})] 
         \leq \exp( c\sum_{k=t+1}^{T} \frac{2\mu_k}{k}).
    \end{align*}
    Note that 
    \begin{align*}
        \mu_k \leq \min\{ \E_\A[\E_{S,z}[\xi_k(S,z,i,\A)]], \beta\}.
    \end{align*}
    Assuming in addition that $c\leq \frac{1}{\beta}$, 
    we have 
    $\E_{S,z}[\xi_k(S,z,i,\A)]$ is bounded by $\gamma$
    by Lemma \ref{lem:non-convex-xi}.
    We have that
    \begin{align*}
        &\E_{S,z}[\Delta_{T+1}(S,z,i)]\\
        &\leq \sum_{t=t_0}^{T}\exp(2c\gamma (1-\frac{1}{n})\sum_{k=t+1}^{T} \frac{1}{k})  (\frac{2c L}{n t}+\frac{c\eta}{t})\\
        &\leq \sum_{t=t_0}^{T}\exp(2c\gamma (1-\frac{1}{n})\ln\frac{T}{t}) (\frac{2c L}{n t}+\frac{c\eta}{t})\\
        &\leq \sum_{t=t_0}^{T}\exp(2c\gamma \ln\frac{T}{t}) (\frac{2c L}{n t}+\frac{c\eta}{t})\\
        &=  (\frac{2c L}{n}+c\eta) T^{2c\gamma} \sum_{t=t_0}^{T} t^{-2c\gamma - 1} \\
        &\leq (\frac{2L+\eta n}{2n\gamma}) (\frac{T}{t_0-1})^{2c\gamma}.
    \end{align*}
    Thus,
    \begin{align}
        &\E_{S, z,\A}[h(\theta_{T+1}, z)-h(\theta_{T+1}', z)]
        \leq (\frac{2L^2+\eta n L}{2n\gamma})(\frac{T}{t_0-1})^{2c\gamma}
        +\frac{t_0-1}{n}\E_{S,\A}[\AR_\Dd(\theta_T)]. \label{equ:non-convex-to-minimize}
    \end{align}
    Let $q= 2c\gamma$ and  $r = \E_{S,\A}[\AR_\Dd(\theta_{T+1})]$.
    Setting
    \begin{align*}
        t_0 = ((2L^2 + n\eta L)\frac{c T^q}{r}  )^{\frac{1}{1+q}} + 1
    \end{align*}
    minimizes Equation~\eqref{equ:non-convex-to-minimize}
    and we have
    \begin{align*}
        \varepsilon(\Dd, \theta_1)
        \leq \frac{1+1/q}{n}(2c L^2+ n c \eta L)^{\frac{1}{1+q}} (Tr)^{\frac{q}{1+q}}.
    \end{align*}
    
\end{proof}


\subsection{Proof of Theorem \ref{thm:main-non-convex-multiple-pass}}
\begin{proof}
Let $\Delta_t(S,z,i)=\E_\A[\delta_t(S,z,i,\A)|\delta_{t_0}(S,z,i,\A)=0]$.
By Lemma \ref{lem:kuz-average-stable-delta}, $\forall t_0\in[n+1]$,
\begin{align*}
&\E_{S, z,\A}[h(\theta_{T+1}, z)-h(\theta_{T+1}', z)]
\leq L \E_{S,z}[\Delta_{T+1}(S,z,i)]+\frac{t_0-1}{n}\E_{S,\A}[\AR_\Dd(\theta_{T+1})].
\end{align*}
When $\pi(t)=i$ with probability $\frac{1}{n}$, we have
\begin{align*}
    ||\G_\A(\theta_t) - \G_\A(\theta'_t)||\leq \delta_t(S,z,i,\A) + 2\alpha_t L.
\end{align*}
When $\pi(t)\neq i$ with probability $1-\frac{1}{n}$,
we have
\begin{align*}
    ||\G_\A(\theta_t) - \G_\A(\theta'_t)||\leq (1+\alpha_t \beta)\delta_t(S,z,i,\A)+\alpha_t \eta,
\end{align*}
by the second statement in Lemma \ref{lem:ref-xiao-approx-grad-lip}.
We have
\begin{align*}
    &\Delta_{t+1}(S,z,i)\\
    \leq& \frac{1}{n} (\Delta_t(S,z,i)+2\alpha_t L)
    + (1-\frac{1}{n}) ((1+\alpha_t \beta)\Delta_t(S,z,i) + \alpha_t \eta)\\
    =& (1+(1-\frac{1}{n})\alpha_t \beta)\Delta_t(S,z,i) + \frac{2\alpha_t L + (n-1)\alpha_t \eta}{n}
    \leq \exp((1-\frac{1}{n})\alpha_t \beta)\Delta_t(S,z,i) + \frac{2\alpha_t L}{n}+\alpha_t \eta.
\end{align*}
Let $\alpha_t = \frac{c}{t}$ with $c\leq \frac{1}{\beta}$.
It follows that
\begin{align*}
    \Delta_{T+1}(S,z,i)
    &\leq \sum_{t=t_0}^{T}(\prod_{k=t+1}^{T} \exp(\frac{(n-1)c\beta)}{nk})) (\frac{2c L}{n t}+\frac{c\eta}{t})\\
    &=\sum_{t=t_0}^{T}\exp(\frac{(n-1)c}{n}\sum_{k=t+1}^{T} \frac{\beta}{k}) (\frac{2c L}{n t}+\frac{c\eta}{t})\\
    &\leq \sum_{t=t_0}^{T}\exp(\frac{(n-1)c\beta}{n}\ln{\frac{T}{t}}) (\frac{2c L}{n t}+\frac{c\eta}{t})\\
    &\leq \sum_{t=t_0}^{T}\exp(2c\beta \ln\frac{T}{t}) (\frac{2c L}{n t}+\frac{c\eta}{t})\\
    &=  (\frac{2c L}{n}+c\eta) T^{2c\beta} \sum_{t=t_0}^{T} t^{-2c\beta - 1} \\
    &\leq (\frac{2L+\eta n}{2n\beta}) (\frac{T}{t_0-1})^{2c\beta}.
\end{align*}

Let $q= 2c\beta$ and  $r = \E_{S,\A}[\AR_\Dd(\theta_{T+1})]$.
Setting
\begin{align*}
    t_0 = ((2L^2 + n\eta L)\frac{c T^q}{r}  )^{\frac{1}{1+q}} + 1,
\end{align*}
we have
\begin{align*}
    \varepsilon(\Dd, \theta_1)
    \leq \frac{1+1/q}{n}(2c L^2+ n c \eta L)^{\frac{1}{1+q}} (Tr)^{\frac{q}{1+q}}.
\end{align*}
\end{proof}

\section{Discussions of uniform stability-based counterparts}\label{app:discussion-of-uniform-stability-based-results}
Uniform stability analysis employs the following notion:
\begin{definition}[Uniform stability]
    A randomized algorithm $\A$ is {\em $\varepsilon$-uniformly stable} if for all $S,S' \in\D^n$ such that $S$ and $S'$ differ in at most one element, we have
    \begin{align}
        \sup_{z\in \D} \E_\A[h(\A(S),z)-h(\A(S'),z)] \leq \varepsilon.\label{equ:uniform-stability}
    \end{align}
\end{definition}
Thus uniform stability bounds the expected difference between the losses of algorithm outputs on two adjacent training sets.
The uniform stability is distribution-free since $S$ and $S'$ are independent of $\Dd$.
A generalization bound was given as follows.
\begin{theorem}[\citealt{hardt2016train}]\label{thm:hardt-uniform-stable}
{If $\A$ is $\varepsilon$-uniformly stable, then the robust generalization {gap} of $\A$ is bounded by $\varepsilon$:}
\begin{align*}
        |\E_{S,\A}[\AR_\Dd(\A(S))- \AR_S(\A(S))]| \leq \varepsilon.
    \end{align*}
\end{theorem}


The following theorem gives an upper bound of the robust generalization gap.
\begin{theorem}[\citealt{xiao2022stability}]
\label{thm:xiao-convex}
    Let $h(\theta, z)$ be convex, $L$-Lipschitz and $\eta$-approximately $\beta$-gradient Lipschitz in $\theta$. 
    Let the step sizes $\alpha_t=\alpha \leq \frac{1}{\beta}$.
    Then the generalization gap of algorithm $\A$
     using the training set of size $n$ after $T$ steps of SGD update  has an upper bound 
    \begin{align*}
        \varepsilon_{\gen} = (\eta + \frac{2L}{n})\alpha T L,
    \end{align*}
    where  $\eta$ is a parameter proportional to the adversarial training budget $\epsilon$ and the approximate gradient Lipschitz condition will be defined in Lemma \ref{lem:h-lip-conditions}.  
\end{theorem}

\subsection{Data poisoning}
A poisoned algorithm $\A_{\P}$ uses the gradient on a poisoned datum 
$\nabla h(\theta, {\P}(z))$ instead of $\nabla h(\theta, z)$
which may result in totally different update trajectories dependent on ${\P}$.
Nevertheless, the expansion properties of $\G_{\A_{\P}}$ are not affected by the attack ${\P}$ at all.
Indeed, due to
\begin{align*}
    \G_{\A_{\P}}(\theta, z, \alpha) = \G_\A(\theta,{\P}(z), \alpha),
\end{align*}
if $\G_\A$ is $\iota$-approximately $\kappa$-expansive,
the poisoned update rule $\G_{\A_{\P}}$ is $\iota$-approximately $\kappa$-expansive as well.
Therefore, uniform stability analysis provides the same upper bound of $||\A_{\P}(S)-\A_{\P}(S')||$. 
Thus, Theorem \ref{thm:xiao-convex} implies the following proposition.
\begin{proposition}
\label{cor:hardt-convex}
    The poisoned generalization gap $\varepsilon_{\P}$ based on the uniform stability analysis
    remains unchanged, i.e.
    \begin{align*}
        \varepsilon_{\P} = (\eta + \frac{2L}{n})\alpha T L.
    \end{align*}
\end{proposition}
Similar consequences also hold for the results in \cite{hardt2016train, xing2021algorithmic}, 
{and the results for non-convex and strongly-convex cases in \cite{xiao2022stability}}.
%



\begin{proof}
    By Theorem \ref{thm:hardt-uniform-stable} and the Lipschitz assumption,
    it suffices to prove
    \begin{align}
        \E_{\A_P}[||\A_P(S) - \A_P(S')||]\leq (\eta + \frac{2L}{n})\alpha T. \label{equ:corollary-suffice}
    \end{align}
    Assume the trajectories of $\A_P(S)$ and $\A_P(S')$ are 
    $\theta_1,\cdots,\theta_T$ and $\theta'_1,\cdots,\theta'_T$ respectively.
    Let $\delta_t = ||\theta_t - \theta'_t||$.
    By the third statement in Lemma \ref{lem:ref-xiao-approx-grad-lip}, 
    the update rule $\G_\A$ is $\alpha \eta$-approximately $1$-expansive.
    Since we have
    \begin{align*}
        \G_{\A_P}(\theta, z, \alpha) = \G_\A(\theta,P(z), \alpha),
    \end{align*}
    the update rule $\G_{\A_P}$ is $\alpha \eta$-approximately $1$-expansive as well due to the Definition \ref{def:expansion}.
    Note that at step $t$ the algorithm $\A_P$ selects the example that $S$ and $S'$ differ with probability $\frac{1}{n}$.
    In this case, 
    \begin{align*}
        \delta_{t+1} \leq \delta_t + 2\alpha L.
    \end{align*}
    In the other case,
    \begin{align*}
        \delta_{t+1} \leq \delta_t + \alpha \eta.
    \end{align*}
    It follows that
    \begin{align*}
        \E_{\A_P}[\delta_{t+1}]
        &\leq \frac{1}{n}(\E_{\A_P}[\delta_t] + 2\alpha L) 
        + (1-\frac{1}{n})(\E_{\A_P}[\delta_t] + \alpha \eta)\\
        &\leq \E_{\A_P}[\delta_t] + (\eta + \frac{2 L}{n})\alpha.
    \end{align*}
    Therefore, Equation~\eqref{equ:corollary-suffice} follows.
\end{proof}

\section{Experiments Setups}\label{app:additional-experiments}

We conduct experiments by adversarially training 
a ResNet-18 \cite{he2016deep} on common datasets and their poisoned counterparts under different data poisoning attacks. 

\subsection{Data Augmentation}
For CIFAR-10 and CIFAR-100, we perform RandomHorizontalFlip,
RandomCrop(32, 4) on the training set 
and Normalize on both the training set and test set. 
For Tiny-ImageNet, we perform RandomHorizontalFlip on the training set 
and Normalize on both the training set and test set.
For SVHN, we perform only Normalize on both the training set and test set.

\subsection{Adversarial Training}
We use the cross entropy loss as the original loss $l$.
We adopt the 10-step projection gradient descent (PGD-10) \cite{madry2017towards} to generate adversarial examples.
The adversarial budget is $\epsilon$ and the step size is $\epsilon/4$ in $L_\infty$-norm.  
We report robust accuracy as the ratio of correctly classified adversarial examples generated by PGD-10,
and the robust generalization gap 
as the gap between robust training accuracy 
and robust test accuracy.

We use the SGD optimizer in PyTorch 
and set the momentum and weight decay to be $0.9$ and $5\times 10^{-4}$ respectively.
For all four data sets,
the batch sizes of data loaders are set to 128.    
In Figure \ref{subfig:svhn-budget} and \ref{subfig:cifar10-budget},
to illustrate the robust overfitting phenomenon, 
we run AT for 200 epochs with an initial learning rate of 0.1 
that decays by a factor of 0.1 at the 100 and 150 epochs.
In other experiments, 
we adopt the constant learning rate of $0.01$.
We run AT for 50 epochs on SVHN 
and for 200 epochs on CIFAR-10, CIFAR-100 and Tiny-ImageNet.   

\subsection{Poisoning Details}
We introduce different poisoning attacks used in our experiments on CIFAR-10 and CIFAR-100. 
In order to simulate the poisoned distribution $\P_\#\Dd$, 
we generate the poisoned training set and test set simultaneously.

\textbf{EM (error minimizing noise).}
\cite{huang2021unlearnable} proposed a min-min bi-level optimization 
to generate error-minimizing noises on the training set. 
Such noises prevent deep learning models from learning information about the clean distribution from the poisoned training data. 
Formally:
\begin{align*}
    \min_{\theta}\frac{1}{n}\sum_{i=1}^{n} \min_{||\delta_i||\le \epsilon'}l(f_\theta(x_i+\delta_i), y_i)),
\end{align*}
where $\epsilon'$ is the poisoning budget.
The trained noise generator $f_\theta$ generates 
an unlearnable example $(x_{em}, y)$ with respect to the clean datum $(x, y)$ such that $x_{em} = x + \arg\min_{||\delta_i||\le \epsilon'}l(f_{\theta}(x+\delta), y)$.
PGD-10 is employed for solving the minimization problem. 
It is worth noting that in our experiments, we combine the training set and test set together to train noise generator $f_\theta$ in order to obtain the poisoned training set and poisoned test set coming from the same shifted distribution.

\textbf{REM (robust error minimizing noise).}
\cite{fu2022robust} further proposed robust minimizing noise 
in order to protect data from adversarial training, 
which also can degrade the test robustness. 
Formally:
\begin{align*}
    \min_{\theta} \frac{1}{n} \sum_{i=1}^n \min_{||\delta_i^{u}|| \le \epsilon'} \max_{||\delta_i^{a}|| \le \rho_a}l(f_{\theta}(x+\delta_i^{u}+ \delta_i^{a}), y),
\end{align*}
where $\epsilon'$ and $\rho_a$ are poisoning budget 
and adversarial perturbation budget.
$\rho_a$ controls the protection level against adversarial training.
For REM, we set  $\rho_a=2/255$.
The trained noise generator $f_\theta$ generates 
a robust unlearnable example $(x_{rem}, y)$ with respect to the clean datum $(x, y)$ such that
$x_{rem} = x + \arg\min_{||\delta_i^{u}|| \le \epsilon'} \max_{||\delta_i^{a}|| \le \rho_a}l(f_{\theta}(x+\delta^{u}+ \delta^{a}), y)$. It is worth noting that in our experiments, we combine the training set and test set together to train noise generator $f_\theta$ in order to obtain the poisoned training set and poisoned test set coming from the same shifted distribution.

Following \cite{fu2022robust}, the source model is trained with SGD for 5000 iterations, 
with batch size of 128,  momentum of 0.9, weight decay of $5\times 10^{-4}$, 
an initial learning rate of 0.1, 
and a learning rate scheduler that decays the learning rate by a factor of 0.1 every 2000 iterations. 
The inner minimization and maximization use PGD-10 to approximate. 
For EOT, the data transformation T is set as the data augmentation of the corresponding data set, 
and the repeated sampling number for expectation estimation is set as 5.

\textbf{ADV (adversarial perturbation).}
\cite{tao2021better} and \cite{fowl2021adversarial} both proposed that 
adding adversarial perturbations to the training data 
is effective to degrade the test performance of a naturally trained model. 
In our experiments, we follow their class-targeted adversarial attack. 
Let $K$ be the number of data classes.
We choose fixed target permutation $t = (y + 1) \mod K$ according to source label $y$. 
Then add a small adversarial perturbation to $x$ in order to force a naturally trained model to classify it as the wrong label $t$. 
Formally:
\begin{align*}
    x_{adv} = \arg\min_{||\delta|| \le \epsilon'} l(f_\theta(x+\delta), t),
\end{align*}
where $f_\theta$ is a classifier naturally trained on the combination of the training set and test set, 
and $\epsilon'$ is the poisoning budget.
For the minimization problem,
we adopt PGD-100 which is enough to generate strong poisons. 

\textbf{HYP (hypocritical perturbation).}
\cite{tao2022can} proposed adding hypocritical perturbation on training data to degrade the test robustness of an adversarially trained model. 
Before generating the poisons, 
a crafting model is adversarially trained with a crafting budget $\epsilon=2/255$ for 10 epochs. 
Then generate hypocritical noises within the perturbation budget $\epsilon'$ 
which can mislead the learner by reinforcing the non-robust features. 
Formally:
\begin{align*}
    x_{hyp} = \arg \min_{||\delta|| \le\epsilon'}l(f_\theta(x+\delta), y),
\end{align*}
where $f_\theta$ is the crafting model. 
Like in ADV, we choose PGD-100 to solve the minimization problem. 
It is worth mentioning that we trained the crafting model on the combination of the training set and test set.

\textbf{RAN (class-wise random perturbation).}
In our experiment, we generate a random perturbation $p_y \in B(0, \epsilon')$ for each label $y$ according to the uniform distribution. 
Then we have poisoned pairs $(x+p_y, y)$. 
It is important that we choose the same class-wise random perturbation for the training set and test set.

\end{document}